\documentclass{article}

\usepackage{microtype}
\usepackage{graphicx}
\usepackage{subfigure}
\usepackage{booktabs} % for professional tables
\usepackage{xcolor}

\usepackage{tcolorbox}

\usepackage{hyperref}

\usepackage[accepted]{icml2024}

\usepackage{amsmath}
\usepackage{amssymb}
\usepackage{mathtools}
\usepackage{amsthm}

\usepackage[capitalize,noabbrev]{cleveref}
\usepackage{thmtools,thm-restate}
\theoremstyle{plain}
\newtheorem{theorem}{Theorem}[section]

\newtheorem{lemma}[theorem]{Lemma}

\theoremstyle{definition}

\newtheorem{assumption}[theorem]{Assumption}
\theoremstyle{remark}

\usepackage[textsize=tiny]{todonotes}
\newif\ifspacehack
\usepackage{natbib}
\hypersetup{
    colorlinks = true,
    breaklinks,
    linkcolor = blue,
    citecolor = blue,
    urlcolor  = blue,
}
\usepackage{url} 
\usepackage{graphicx}
\usepackage{mathtools}
\usepackage{footnote}
\usepackage{float}
\usepackage{xspace}
\usepackage{multirow}
\usepackage{xcolor}
\usepackage{wrapfig}
\usepackage{framed}
\usepackage{bbm}
\usepackage{footnote}
\usepackage{nicefrac}
\usepackage{makecell}
\usepackage[algo2e]{algorithm2e} 
\usepackage{algorithm}
\usepackage{amssymb}
\usepackage{cleveref}
\usepackage{bm}
\makesavenoteenv{tabular}
\makesavenoteenv{table}

\renewcommand{\tilde}{\widetilde}

\def \R {\mathbb{R}}
\newcommand{\eps}{\epsilon}

\newcommand{\calA}{{\mathcal{A}}}

\newcommand{\calX}{{\mathcal{X}}}

\newcommand{\calF}{{\mathcal{F}}}

\newcommand{\calG}{{\mathcal{G}}}

\newcommand{\calQ}{{\mathcal{Q}}}

\newcommand{\calN}{{\mathcal{N}}}

\newcommand{\cF}{{\mathcal{F}}}

\newcommand{\E}{{\mathbb{E}}}

\newcommand{\order}{\mathcal{O}}

\DeclareMathOperator*{\argmin}{argmin}
\DeclareMathOperator*{\argmax}{argmax}

\newcommand{\wh}{\widehat}

\DeclareMathOperator{\conv}{conv}

\newcommand{\squareCB}{\ensuremath{\mathsf{SquareCB}}\xspace}

\newcommand{\graphCB}{\ensuremath{\mathsf{SquareCB.G}}\xspace}
\newcommand{\graphCBp}{\ensuremath{\mathsf{SquareCB.UG}}\xspace}

\newcommand{\AlgSq}{\ensuremath{\mathsf{AlgSq}}\xspace}
\newcommand{\AlgLog}{\ensuremath{\mathsf{AlgLog}}\xspace}

\newcommand{\dec}{\ensuremath{\mathsf{dec}}\xspace}
\newcommand{\decp}{\ensuremath{\mathsf{dec}^{\mathsf{P}}}\xspace}

\newcounter{numrellocal}% Local counter for numering relations
\newcounter{numrelglobal}% Global counter for numering relations
\AfterEndEnvironment{align*}{\setcounter{numrellocal}{0}}% Resets numrellocal at the end of align*

\newcommand{\paren}[1]{\left({#1}\right)}

\newcommand{\otil}{\ensuremath{\tilde{\mathcal{O}}}}

\usepackage{lipsum,booktabs}
\usepackage{amsmath,mathrsfs,amssymb,amsfonts,bm,enumitem}
\usepackage{rotating}
\usepackage{pdflscape}
\usepackage{hyperref,url}
\hypersetup{
    colorlinks,
    breaklinks,
    linkcolor = blue,
    citecolor = blue,
    urlcolor  = blue,
}
\allowdisplaybreaks
\usepackage{appendix}
\usepackage{multirow,makecell}

\usepackage{algorithmic,algorithm}
\renewcommand{\algorithmicrequire}{ \textbf{Input:}}
\renewcommand{\algorithmicensure}{ \textbf{Output:}}

\renewcommand{\tilde}{\widetilde}

\def \E {\mathbb{E}}

\def \R {\mathbb{R}}

\newcommand{\RegSq}{\ensuremath{\mathrm{\mathbf{Reg}}_{\mathsf{Sq}}}\xspace}
\newcommand{\RegCB}{\ensuremath{\mathrm{\mathbf{Reg}}_{\mathsf{CB}}}\xspace}

\newcommand{\RegLogG}{\ensuremath{\mathrm{\mathbf{Reg}}_{\mathsf{Log}}}\xspace}
\newcommand{\RegLogGF}{\ensuremath{\mathrm{\mathbf{Reg}}_{\mathsf{Log}}}\xspace}

\usepackage{mathtools}
\usepackage{graphicx,color} % more modern

\definecolor{wine_red}{RGB}{228,48,64}
\definecolor{DSgray}{cmyk}{0,1,0,0}

\usepackage{prettyref}
\newcommand{\pref}[1]{\prettyref{#1}}

\newcommand{\savehyperref}[2]{\texorpdfstring{\hyperref[#1]{#2}}{#2}}
\newrefformat{eq}{\savehyperref{#1}{Eq. \textup{(\ref*{#1})}}}
\newrefformat{eqn}{\savehyperref{#1}{Eq.~(\ref*{#1})}}
\newrefformat{lem}{\savehyperref{#1}{Lemma~\ref*{#1}}}
\newrefformat{def}{\savehyperref{#1}{Definition~\ref*{#1}}}
\newrefformat{line}{\savehyperref{#1}{Line~\ref*{#1}}}
\newrefformat{thm}{\savehyperref{#1}{Theorem~\ref*{#1}}}
\newrefformat{corr}{\savehyperref{#1}{Corollary~\ref*{#1}}}
\newrefformat{cor}{\savehyperref{#1}{Corollary~\ref*{#1}}}
\newrefformat{sec}{\savehyperref{#1}{Section~\ref*{#1}}}
\newrefformat{app}{\savehyperref{#1}{Appendix~\ref*{#1}}}
\newrefformat{assum}{\savehyperref{#1}{Assumption~\ref*{#1}}}
\newrefformat{asm}{\savehyperref{#1}{Assumption~\ref*{#1}}}
\newrefformat{ex}{\savehyperref{#1}{Example~\ref*{#1}}}
\newrefformat{fig}{\savehyperref{#1}{Figure~\ref*{#1}}}
\newrefformat{alg}{\savehyperref{#1}{Algorithm~\ref*{#1}}}
\newrefformat{rem}{\savehyperref{#1}{Remark~\ref*{#1}}}
\newrefformat{conj}{\savehyperref{#1}{Conjecture~\ref*{#1}}}
\newrefformat{prop}{\savehyperref{#1}{Proposition~\ref*{#1}}}
\newrefformat{proto}{\savehyperref{#1}{Protocol~\ref*{#1}}}
\newrefformat{prob}{\savehyperref{#1}{Problem~\ref*{#1}}}
\newrefformat{claim}{\savehyperref{#1}{Claim~\ref*{#1}}}
\newrefformat{que}{\savehyperref{#1}{Question~\ref*{#1}}}
\newrefformat{op}{\savehyperref{#1}{Open Problem~\ref*{#1}}}
\newrefformat{fn}{\savehyperref{#1}{Footnote~\ref*{#1}}}

\def \epsilon {\varepsilon}

\newcommand{\pa}[1]{\left(#1\right)}

\icmltitlerunning{Efficient Contextual Bandits with Uninformed Feedback Graphs}

\begin{document}

\twocolumn[
\icmltitle{Efficient Contextual Bandits with Uninformed Feedback Graphs}




\begin{icmlauthorlist}
\icmlauthor{Mengxiao Zhang}{usc}
\icmlauthor{Yuheng Zhang}{uiuc}
\icmlauthor{Haipeng Luo}{usc}
\icmlauthor{Paul Mineiro}{msr}
\end{icmlauthorlist}

\icmlaffiliation{usc}{University of Southern California}
\icmlaffiliation{uiuc}{University of Illinois Urbana-Champaign}
\icmlaffiliation{msr}{Microsoft Research}

%\icmlcorrespondingauthor{Firstname1 Lastname1}{first1.last1@xxx.edu}
%\icmlcorrespondingauthor{}{}

\icmlkeywords{Machine Learning, ICML}

\vskip 0.3in
]



\printAffiliationsAndNotice{} % otherwise use the standard text.

\begin{abstract}
Bandits with feedback graphs are powerful online learning models that interpolate between the full information and classic bandit problems, capturing many real-life applications.
A recent work by \citet{zhang2023practical} studies the contextual version of this problem and proposes an efficient and optimal algorithm via a reduction to online regression.
However, their algorithm crucially relies on seeing the feedback graph before making each decision, while in many applications, the feedback graph is \emph{uninformed}, meaning that it is either only revealed after the learner makes her decision or even never fully revealed at all.
This work develops the first contextual algorithm for such uninformed settings, via an efficient reduction to online regression over both the losses and the graphs.
Importantly, we show that it is critical to learn the graphs using \emph{log loss} instead of squared loss to obtain favorable regret guarantees. 
We also demonstrate the empirical effectiveness of our algorithm on a bidding application using both synthetic and real-world data.
\end{abstract}

\section{Introduction}\label{sec:intro}
In this paper, we consider efficient algorithm design for contextual bandits with directed feedback graphs, which generalizes classic contextual bandits~\citep{auer2002nonstochastic, langford2007epoch}. The interaction between the learner and the environment lasts for $T$ rounds. At each round, the learner first observes a context and then chooses one of $K$ actions, while simultaneously an adversary decides the loss for each action and a \emph{directed feedback graph} with the $K$ actions as nodes. After that, the learner suffers the loss of the chosen action, and her observation is determined based on this feedback graph. Specifically, she observes the loss of every action to which the chosen action is connected. This problem reduces to the classic contextual bandits when the feedback graph only contains self-loops, and generally captures more real-world applications, such as personalized web advertising~\citep{mannor2011bandits}:

The non-contextual version of this problem has been studied extensively in the literature~\citep{mannor2011bandits,alon2015online,alon2017nonstochastic}. \citet{alon2015online} provided a full characterization on the minimax regret rate with respect to different graph theoretic quantities 
according to different types of the feedback graphs. However, the more practically useful contextual version is much less explored. A recent work by~\citet{zhang2023practical} first considered this problem with adversarial context and general feedback graphs, 
and proposed an efficient algorithm achieving minimax regret rates. However, their algorithm only works in the \emph{informed} setting where the feedback graph is revealed to the learner before her decision. 
In many applications, such as online pricing~\citep{cohen2016online}, the feedback graph is not available to the learner when (or even after) she makes the decision. \citet{cohen2016online} first considered this more challenging \emph{uninformed} setting without context and derived algorithms achieving near-optimal regret guarantees. However, there is no prior work offering a solution for contextual bandits with uninformed feedback graphs. In this work, we take the first step in this direction and propose the first efficient algorithms achieving strong regret guarantees. 
Specifically, our contributions are as follows.

\paragraph{Contributions.} Our algorithm, \graphCBp, is based on the \graphCB algorithm of~\citet{zhang2023practical}. Assuming realizability on the loss function and an online square loss regression oracle, \citet{zhang2023practical} extended the minimax framework for contextual bandits~\citep{foster2020beyond,foster2021efficient} to contextual bandits under informed feedback graphs. With uninformed graphs, we further assume that they are realizable by another function class $\calG$ and propose to learn them simultaneously so that in each round we can plug in the predicted graph into \graphCB.
While the idea is natural, our analysis is highly non-trivial and, perhaps more importantly, reveals that it is crucial to learn these graphs using \emph{log loss} instead of squared loss.

More specifically, within the uninformed setting, 
we analyze two different types of feedback on the graph structure. In the \emph{partially revealed graph} setting, 
the learner only observes which actions are connected to the selected action, and our algorithm achieves $\otil(\sqrt{\alpha(\calG) T})$ regret (ignoring the regression overhead), where $\alpha(\calG)$ is the maximum expected independence number over all graphs in $\calG$; in the easier \emph{fully revealed graph setting}, 
the learner observes the entire graph after her decision, and our algorithm achieves an improved $\otil\Big(\sqrt{\sum_{t=1}^T\alpha_t}\Big)$ regret bound, where $\alpha_t$ denotes the expected independence number of the feedback graph at round $t$.\footnote{Definitions of all graph-theoretic numbers are formally introduced in \pref{sec:pre}. Also, for simplicity this paper only considers strongly observable graphs where $\sqrt{T}$-regret is achievable, but our ideas can be directly generalized to weakly observable graphs as well (where $T^{2/3}$-regret is minimax optimal). \label{fn:strong_vs_weak}} 
We note that this latter bound even matches the optimal regret for the easier informed setting~\citep{zhang2023practical}.

In addition to these strong theoretical guarantees, we also empirically test our algorithm on a bidding application with both synthetic and real-world data and show that it indeed outperforms the greedy algorithm or algorithms that ignore the additional feedback from the graphs.

\paragraph{Related works.} 

(Non-contextual) multi-armed bandits with feedback graphs was first studied by \citet{mannor2011bandits}. \citet{alon2015online} characterized the minimax rates in terms of graph-theoretic quantities under deterministic feedback graphs and proposed algorithms achieving near-optimal guarantees. 
Since then, many different extensions have been studied, including stochastic feedback graphs~\citep{kocak2016online,liu2018information,li2020stochastic,esposito2022learning}, uninformed feedback graphs~\citep{cohen2016online,esposito2022learning}, algorithms that adapt to both adversarial and stochastic losses~\citep{erez2021towards,ito2022nearly,dann2023blackbox}, data-dependent regret
bounds~\citep{lykouris2018small,lee2020closer}, and high-probability regret~\citep{neu2015explore,luo2023improved}.

The contextual version of this problem has only been studied very recently.
\citet{wang2021adversarial} developed algorithms for adversarial linear bandits with uninformed graphs and stochastic contexts, but assumed several strong assumptions on both the policy class and the context space. \citet{zhang2023practical} is the closest to our work. They also considered adversarial contexts and realizable losses, 
but as mentioned, their algorithm is only applicable to the informed setting.
Moreover, their theoretical results are also restricted to deterministic feedback graphs only.
Our work generalizes theirs from the informed setting to the uninformed setting and from deterministic graphs to stochastic graphs.

Our work is also closely related to the recent trend of designing efficient algorithms for contextual bandits. Since \citet{langford2007epoch} initiated the study of efficient learning in contextual bandits, many follow-ups develop efficient contextual bandits algorithms via reduction to either cost-sensitive classification~\citep{dudik2011efficient,agarwal2014taming} or online/offline regression~\citep{foster2020beyond,foster2021efficient,foster2021statistical,xu2020upper,simchi2022bypassing}. We follow the latter approach and reduce our problem to online regression on both the losses and the feedback graphs.
\section{Preliminary}\label{sec:pre}
Throughout the paper, we denote the set $\{1,2,\dots,m\}$ for some positive integer $m$ by $[m]$,
the set of distributions over some set $S$ by $\Delta(S)$,
and the convex hull of some set $S$ by $\conv(S)$.
For a vector $v\in \R^m$ and a matrix $M\in \R^{m\times m}$, $v_i$ denotes the $i$-th coordinate of $v$ and $M_{i,j}$ denotes the $(i,j)$'s entry of $M$ for $i,j\in[m]$.

The contextual bandits problem with uninformed feedback graphs proceeds
in $T$ rounds. 
At each round $t$, the environment (possibly randomly and adaptively) selects a context $x_t$ from some arbitrary context space $\calX$, a loss vector $\ell_t\in [0,1]^{K}$ specifying the loss of each of the $K$ possible actions, and finally a directed feedback graph $G_t=([K], E_t)$ where $E_t\subseteq [K]\times[K]$ denotes the set of directed edges.
The learner then observes the context $x_t$ (but not $\ell_t$ or $G_t$) and has to select an action $i_t\in [K]$.
At the end of this round, the learner suffers loss $\ell_{t,i_t}$ and observes the loss of every action connected to $i_t$ (not necessarily including $i_t$ itself): 
$\ell_{t,j}$ for all $j\in A_t$, where $A_t=\{j\in[K], (i_t,j)\in E_t\}$. 
In the \emph{partially revealed graph} setting, the learner does not observe anything else about the graph (other than $A_t$), while in the \emph{fully revealed graph} setting, the learner additionally observes the entire graph (that is, $E_t$). 

\citet{alon2015online} showed that in the non-contextual version of this problem, there are essentially only two types of nontrivial and learnable feedback graphs: strongly observable graphs and weakly observable graphs.
For simplicity, our work focuses soly on the first type, that is, we assume that $G_t$ is always strongly observable, meaning that for each node $i\in[K]$, either it can observe itself ($(i,i) \in E_t$) or it can be observed by any other nodes ($(j,i) \in E_t$ for any $j\in[K]\backslash\{i\}$).
As mentioned in \pref{fn:strong_vs_weak}, our results can be directly generalized to weakly observable graphs as well.

Bandits with uninformed feedback graphs naturally capture many applications such as online pricing, viral marketing, and recommendation in social networks~\citep{kocak2014efficient,alon2015online,alon2017nonstochastic,rangi2019online,cohen2016online,liu2018information}.
By incorporating contexts, which are broadly available in practice, our model significantly increases its applicability in real world.
For a concrete example, see \pref{sec:experiment} for an application of bidding in a first-price auction.

\paragraph{Realizability and oracle assumptions.}
Following a line of recent works on developing efficient contextual bandit algorithms, we make the
following standard realizability assumption on the loss function, stating that the expected loss of each action can be perfectly predicted by an unknown loss predictor from a known class:

\begin{assumption}[Realizability of mean loss]\label{assum:loss}
    We assume that the learner has access to a function class $\calF\subseteq (\calX \times [K] \mapsto[0, 1])$ in which there exists an unknown regression function $f^\star\in\calF$ such that for any $i\in[K]$ and $t\in [T]$, we have $\E[\ell_{t,i}~|~x_t]=f^\star(x_t,i)$.
\end{assumption}

The goal of the learner is naturally to be comparable to an oracle strategy that knows $f^\star$ ahead of time, formally measured by the (expected) regret:
\begin{align*}
\RegCB\triangleq\E\left[\sum_{t=1}^T (f^\star(x_t,i_t)-\min_{i\in[k]}f^\star(x_t,i) )\right].
\end{align*}

To efficiently minimize regret for a general class $\calF$, it is important to assume some oracle access to this class. 
To this end, we follow prior works and assume that the learner is given an \emph{online regression oracle} \AlgSq for function
class $\cF$, which follows the following protocol: at each round $t \in [T]$, the oracle \AlgSq produces an estimator $f_t\in\conv(\calF)$, then
receives a context $x_t$ and a set $S_t$ of action-loss pairs in the form $(a, c) \in [K] \times [0,1]$.
The squared loss of the oracle for this round is $\sum_{(a,c) \in S_t} (f_t(x_t,a)-c)^2$, which is on average assumed to be close to that of the best predictor in $\cF$:
\begin{assumption}[Bounded squared loss regret]
\label{asm:regression_oracle_loss}
The regression oracle \AlgSq guarantees: 
\begin{align*}
&\sum_{t=1}^T \sum_{(a,c) \in S_t} (f_t(x_t, a) - c)^2
                   \\
                   &\qquad-\inf_{f \in \cF}\sum_{t=1}^T \sum_{(a,c)\in S_t} (f(x_t, a) - c)^2
  \leq \RegSq.
\end{align*}
\end{assumption}

Here, $\RegSq$ is a regret bound that is sublinear in $T$ and depends on some complexity measure of $\cF$; see e.g.~\citet{foster2020beyond} for concrete examples of such oracles and the corresponding regret bounds.
The point is that online regression is such a standard machine learning practice, so reducing our problem to online regression is both theoretically reasonable and  practically desirable. 

So far, we have made exactly the same assumptions as~\citet{zhang2023practical} which studies the informed setting.
In our uninformed setting, however, since nothing is known about the feedback graph before deciding which action to take, we propose to additionally learn the feedback graphs, which requires the following extra realizability and oracle assumptions related to the graphs.
\begin{algorithm*}[t]
\caption{\graphCBp}\label{alg:squareCB.GPLUS}

Input: parameter $\gamma\geq 0$, a regression oracle $\AlgSq$ for loss prediction, and a regression oracle $\AlgLog$ for graph prediction

\For{$t=1,2,\dots,T$}{

    Receive context $x_t$.

    Obtain a loss estimator ${f}_t$ from the oracle $\AlgSq$ and a graph estimator $g_t$ from $\AlgLog$.

    Compute $p_t =\argmin_{p\in \Delta([K])} {\dec}_{\gamma}(p;{f}_t,g_t,x_t)$ (a simple convex program; see \pref{app:implementation}),
    where
    {\small
    \begin{align}\label{eqn:minimax_p_alg}
 {\dec}_{\gamma}(p;f,g,x) &= \sup_{\substack{i^\star \in [K] \\ v^\star\in [0,1]^K}}
	\left[ \sum_{i=1}^Kp_iv^\star_i - v^\star_{i^\star} -\frac{\gamma}{4}\sum_{i=1}^K p_i\sum_{j=1}^K g(x,i,j)\paren{f(x,j) - v^\star_j}^2 \right].
    \end{align}}

    The environment decides loss $\ell_t$ and feedback graph $G_t = ([K], E_t)$.

    The learner samples $i_t$ from $p_t$ and observe $\{ \ell_{t,j} \}_{j \in A_t}$ where $A_t=\{j\in[K]: (i_t,j)\in E_t\}$.

    Feed $x_t$ and $\{(j, \ell_{t,j})\}_{j\in A_t}$ to the oracle \AlgSq.
    
    \textbf{Case 1 (Partially revealed graph):} Construct $S_t=\{(i_t,j,1)\}_{j\in A_t} \cup \{(i_t,j,0)\}_{j\notin A_t}$. 
    
    \textbf{Case 2 (Fully revealed graph):} Observe $E_t$ and construct $S_t=\{(i,j,1)\}_{(i,j)\in E_t} \cup \{(i,j,0)\}_{(i,j)\notin E_t}$. 

    Feed $x_t$ and $S_t$ to the oracle $\AlgLog$.
}
\end{algorithm*}

\begin{assumption}[Realizability of mean graph]\label{assum:graph}
    We assume that the learner has access to a function class $\calG\subseteq (\calX \times [K] \times [K] \mapsto[0, 1])$ in which there exists a regression function $g^\star\in\calG$ such that for any $(i,j)\in[K]\times [K]$ and $t\in [T]$, we have $\E[\mathbbm{1}\{(i,j)\in E_t\}~|~x_t]=g^\star(x_t,i,j)$. 
\end{assumption}

Similarly, since we do not impose specific structures on $\calG$, we assume that the learner can access $\calG$ through the use of another online oracle $\AlgLog$: at each round $t\in[T]$, $\AlgLog$ produces an estimator $g_t\in\conv(\calG)$, then receives a context $x_t$ and a set $S_t$ of tuples in the form $(i,j,b) \in [K]\times[K]\times\{0,1\}$ where $b=1$ ($b=0$) means $i$ is (is not) connected to $j$.
Importantly, our analysis shows that it is critical for this oracle to learn the graphs using \emph{log loss} instead of squared loss (hence the name $\AlgLog$); see detailed explanations in \pref{sec:analysis_partial}.
More specifically, we assume that the oracle satisfies the following regret bound measured by log loss:

\begin{assumption}[Bounded log loss regret]
\label{asm:regression_oracle_graph}
The regression oracle $\AlgLog$ guarantees: 
\begin{align*}
&\sum_{t=1}^T \sum_{(i,j,b)\in S_t}\ell_{\log}(g_t(x_t,i,j), b) \\
&\qquad - \inf_{g\in\calG}\sum_{t=1}^T\sum_{(i,j,b)\in S_t}\ell_{\log}(g(x_t,i,j), b) \leq \RegLogG,
\end{align*}
where for two scalars $u,v\in[0,1]$, $\ell_{\log}(u,v)$ is defined as
\begin{align*}
    \ell_{\log}(u,v) = v\log\frac{1}{u}+(1-v)\log\frac{1}{1-u}.
\end{align*}
\end{assumption}

Once again, the bound $\RegLogG$ is sublinear in $T$ and depends on some complexity measure of $\calG$.
We note that regression using log loss is also highly standard in practice.
For concrete examples, we refer the readers to~\citet{foster2021efficient} where the same log loss oracle was used (for a different purpose of obtaining first-order regret guarantees for contextual bandits).
In our analysis, we also make use of the following important technical lemma that connects the log loss regret with something called the \emph{triangular discrimination} under the realizability assumption.

\begin{restatable}[Proposition 5 of \citet{foster2021efficient}]{lemma}{log_sq_convert}\label{lem:log-sq-convert}
    Suppose for each $t$ and $(i,j,b)\in S_t$, we have $\E[b|x_t] = g^\star(x_t, i,j)$. Then oracle $\AlgLog$ guarantees:
    \begin{align*}
        \E\left[\sum_{t=1}^T\sum_{(i,j,b)\in S_t}\frac{(g_t(x_t,i,j)-g^\star(x_t,i,j))^2}{g_t(x_t,i,j)+g^\star(x_t,i,j)}\right] \leq 2\RegLogG.
    \end{align*}
\end{restatable}

\paragraph{Independence number.} 
It is known that for strongly observable graphs, their \emph{independence numbers} characterize the minimax regret~\citep{alon2015online}.
Specifically, an independence set of a directed graph is a subset of nodes in which no two distinct nodes are connected. The size of the largest independence set in graph $G$ is called its independence number, denoted by $\alpha(G)$. 
Since we consider stochastic graphs, we further define the independence number with respect to a $g\in\conv(\calG)$ and a context $x$ as:
\begin{align*}
    \alpha(g,x)\triangleq \inf_{q\in\calQ(g,x)}\E_{G\sim q}[\alpha(G)],
\end{align*}
where $\calQ(g,x)$ denotes the set of all distributions of strongly observable graphs whose expected edge connections are specified by $g(x, \cdot, \cdot)$ (that is, for any $q\in \calQ(g,x)$, we have $\E_{([K],E)\sim q}\left[\mathbbm{1}\{(i,j)\in E\}\right] = g(x,i,j)$ for all $(i,j)$).
With this notion, the difficulty of $G_t$ is then characterized by the independence number $\alpha_t\triangleq\alpha(g^\star,x_t)$.
In the more challenging partially revealed graph setting, however, our result depends on the worst-case independence number over the entire class $\calG$: $\alpha(\calG)\triangleq\sup_{g\in\calG, x\in\calX}\alpha(g,x)$. 
\section{Algorithms and Regret Guarantees}\label{sec:alg}

In this section, we introduce our algorithm \graphCBp and its regret guarantees. To describe our algorithm, we first briefly introduce the \graphCB algorithm of~\citet{zhang2023practical} for the informed setting: at each round $t$, given the loss estimator $f_t$ obtained from the regression oracle \AlgSq and the feedback graph $G_t$, \graphCB finds the action distribution $p_t\in\Delta([K])$ by solving $\argmin_{p}\dec_\gamma(p;{f}_t,G_t,x_t)$, where
the \emph{Decision-Estimation Coefficient} (DEC) is defined as:
    \begin{align}\label{eqn:minimax_p}
    &\dec_\gamma(p;{f}_t,G_t,x_t) \triangleq \nonumber
    \sup_{\substack{i^\star \in [K] \\ v^\star\in [0,1]^K}}
	 \Bigg[ \sum_{i=1}^Kp_iv^\star_i - v^\star_{i^\star}  \\
 &\qquad - \frac{\gamma}{4} \sum_{i=1}^K p_i\sum_{j=1}^K G_{t,i,j}\left(f_t(x_t,j) -v^\star_j\right)^2  \Bigg]
\end{align}
for some parameter $\gamma>0$,
and we abuse the notation by letting $G_t$ also represent its adjacent matrix.
The idea of DEC originates from~\citet{foster2020beyond} for contextual bandits and has become a general way to tackle interactive decision making problems since then~\citep{foster2021statistical}.
The first two terms within the supremum of \pref{eqn:minimax_p} is the instantaneous regret of strategy $p$ against the best action $i^\star$ with respect to a loss vector $v^\star$, and the third term corresponds to the expected squared loss between the loss predictor $f_t(x_t,\cdot)$ and the loss vector $v^\star$ on the observed actions (since each action $i$ is selected with probability $p_i$ and, conditioning on $i$ being selected, each action $j$ is observed with probability $G_{t,i,j}$). 
Because the true loss vector is unknown, a supremum over $v^\star$ is taken (that is, the worst case is considered).
The goal of the learner is to pick $p_t$ to minimize this DEC, since a small DEC means that the regret suffered by the learner is close to the regret of the regression oracle (which is assumed to be bounded).
After selecting an action $i_t \sim p_t$ and seeing the loss of actions connected $i_t$, \graphCB naturally feeds these observations to \AlgSq and proceeds to the next round.

The clear obstacle of running \graphCB in the uninformed setting is that $G_t$, required in \pref{eqn:minimax_p}, is unavailable at the beginning of round $t$. Therefore, we propose to learn the graphs simultaneously and simply use a predicted graph in place of the true graph $G_t$.
More concretely, at the beginning of round $t$, in addition to the loss estimator $f_t$, we also obtain a graph estimator $g_t$ from the graph regression oracle $\AlgLog$.
Then, to find $p_t$, we solve the same problem but with $G_t$ in \pref{eqn:minimax_p} replaced by the estimator $g_t(x_t, \cdot, \cdot)$; see~\pref{eqn:minimax_p_alg}.
After picking action $i_t \sim p_t$, the training of $\AlgSq$ remains the same, and additionally, we feed all the observations about the structure of $G_t$ to $\AlgLog$: for the partially reveal graph setting, the observations are the connections between $i_t$ and all $j\in A_t$ and the disconnections between $i_t$ and all $j\notin A_t$;
while for the fully reveal graph setting, the observations are all the connections in $E_t$ and the disconnections for all other action pairs.
See \pref{alg:squareCB.GPLUS} for the complete pseudocode.

While the idea of our algorithm appears to be very natural, its analysis is in fact highly non-trivial and reveals that learning the graphs using log loss is critical; see \pref{sec:analysis} for more discussions.
We also note that finding the minimizer of the DEC can be written as a simple convex program and solved efficiently; see more implementation details in \pref{app:implementation}.

We prove the following regret guarantee of \graphCBp in the partially revealed graph setting.
\begin{theorem}\label{thm:partial}
    Under Assumptions~\ref{assum:loss}-\ref{asm:regression_oracle_graph} and partially revealed feedback graphs, \graphCBp with $\gamma=\max\left\{4,\sqrt{\frac{\alpha(\calG)T\log(KT)}{\max\{\RegSq,\RegLogG\}}}\right\}$ guarantees:\footnote{The notation $\otil(\cdot)$  hides logarithmic dependence on $K$ and $T$.}
    \begin{align*}
        \RegCB=\otil\left(\sqrt{\alpha(\calG) T\max\{\RegSq,\RegLogG\}}\right).
    \end{align*}
\end{theorem}

Since $\RegSq$ and $\RegLogG$ are both sublinear in $T$, this regret bound is also sublinear in $T$.
More importantly, it has no polynomial dependence on the total number of actions $K$, and instead only depends on the worst-case independence number $\alpha(\calG) \leq K$.
There are indeed important applications where the independence number of every encountered feedback graph must be small or even independent of $K$ (e.g., the inventory control problem discussed in~\citet{zhang2023practical} or the bidding application in \pref{sec:experiment}), in which case it only makes sense to pick $\calG$ such that $\alpha(\calG)$ is also small.
While \graphCBp requires setting $\gamma$ with the knowledge of $\alpha(\calG)$, in~\pref{app:parameter_free}, we show that applying certain doubling trick on the DEC value leads to the same regret bound even \emph{without} the knowledge of $\alpha(\calG)$.

However, it would be even better if instead of paying $\alpha(\calG)$ every round, we only pay the independence number of the corresponding stochastic feedback graph at each round $t$, that is, $\alpha_t$.
While it is unclear to us whether this is achievable with partially revealed graphs, in the next theorem, we show that \graphCBp indeed achieves this in the easier fully revealed graph setting.

\begin{theorem}\label{thm:fully}
    Under Assumptions~\ref{assum:loss}-\ref{asm:regression_oracle_graph} and fully revealed feedback graphs, \graphCBp with $\gamma=\max\left\{12,\sqrt{\frac{\sum_{t=1}^T\alpha_t}{\max\{\RegSq,\RegLogGF\}}}\right\}$ guarantees:
    \begin{align*}
        \RegCB=\otil\left(\sqrt{\sum_{t=1}^T\alpha_t\max\{\RegSq,\RegLogGF\}}\right).
    \end{align*}
\end{theorem}

In other words, we replace the $\alpha(\calG)T$ term in the regret with the smaller and more adaptive quantity $\sum_{t=1}^T \alpha_t$, indicating that the complexity of learning only depends on how difficult each encountered graph is, but not the worst case difficulty among all the possible graphs in $\calG$.
\section{Analysis}\label{sec:analysis}

In this section, we provide some key steps of our analysis, highlighting 1) why it is enough to replace the true graph in \pref{eqn:minimax_p} with the graph estimator $g_t$;
2) why using log loss in the graph regression oracle is important; and
3) why having fully revealed graphs helps improve the dependence from $\alpha(\calG)$ to $\alpha_t$.

\subsection{Analysis for Partially Revealed Graphs}\label{sec:analysis_partial}

While we present the DEC in \pref{eqn:minimax_p_alg} as a natural modification of \pref{eqn:minimax_p} in the absence of the true graph,
it in fact can be rigorously derived as an upper bound on another DEC more tailored to our original problem.
Specifically, we define for two parameters $\gamma_1, \gamma_2 > 0$:
\begin{align}\label{eqn:dec_graph_partial}
    &\decp_{\gamma_1,\gamma_2}(p;f,g,x) 
    \triangleq \sup_{\substack{i^\star \in [K], v^\star\in [0,1]^K \nonumber \\ M^\star\in [0,1]^{K\times K}}}
	\left[ \sum_{i=1}^Kp_iv^\star_i - v^\star_{i^\star} \right. \nonumber\\
 &\qquad -\gamma_1\sum_{i=1}^K p_i \sum_{j=1}^K M_{i,j}^\star\paren{f(x,j) - v^\star_j}^2 \nonumber \\
 &\qquad \left.- \gamma_2\sum_{i=1}^K p_i \sum_{j=1}^K \frac{(M_{i,j}^\star-g(x,i,j))^2}{M_{i,j}^\star+g(x,i,j)}
 \right].
\end{align}
Similar to \pref{eqn:minimax_p}, the first two terms within the supermum represent the instantaneous regret of strategy $p$ against the best action $i^\star$ with respect to a loss vector $v^\star$.
The third term is also similar and represents the squared loss of \AlgSq under strategy $p$, but since the true graph $G_t$ is unknown, it is replaced with the worst-case adjacent matrix $M^\star$ (hence the supremum over $M^\star$).
Finally, the last additional term is the expected triangular discrimination between $M^\star_{i,\cdot}$ and $g(x,i,\cdot)$ when $i$ is sampled from $p$, which, according to \pref{lem:log-sq-convert}, represents the log loss regret of $\AlgLog$.

Once again, the idea is that if for every round $t$, we can find a strategy $p_t$ with a small DEC value $\decp_{\gamma_1,\gamma_2}(p_t;f_t,g_t,x_t)$, then the learner's overall regret $\RegCB$ will be close to the square loss regret $\RegSq$ of $\AlgSq$ plus the log loss regret $\RegLogG$ of $\AlgLog$, both of which are assumed to be reasonably small. This is formally stated below. 
\begin{restatable}{theorem}{decbound}
\label{thm:decbound}
Under Assumptions~\ref{assum:loss}-\ref{asm:regression_oracle_graph}, for any $\gamma_1, \gamma_2 \geq 0$, the regret $\RegCB$ of \graphCBp is at most
\[\E\left[\sum_{t=1}^T \decp_{\gamma_1,\gamma_2}(p_t;f_t,g_t,x_t)\right] + \gamma_1 \RegSq + 2\gamma_2 \RegLogG.
\]
\end{restatable}

Now, instead of directly minimizing the DEC to find the strategy $p_t$ (which is analytically complicated), the following lemma shows that the easier form of \pref{eqn:minimax_p_alg} serves as an upper bound of \pref{eqn:dec_graph_partial}, further explaining our algorithm design.

\begin{restatable}{lemma}{decTranslation}
\label{lem:dec_translation}
    For any $p\in\Delta([K])$, $g\in \conv(\calG)$, $f\in\conv(\calF)$, and $x\in \calX$, we have 
    \[
    \decp_{\frac{3}{4}\gamma,\frac{1}{4}\gamma}(p;f,g,x)\leq \dec_{\gamma}(p;f,g,x).
    \]
\end{restatable}

To prove this lemma, we need to connect the squared loss with respect to $M^\star$ and that with respect to $g$.
We achieve so using the following lemma, which also reveals why triangular discrimination naturally comes out.

\begin{lemma}\label{lem:z'2z}
    For any $z,z' > 0$, the following holds: \[3z' \geq z - \frac{(z-z')^2}{z+z'}.\]
\end{lemma}
\begin{proof}
    By AM-GM inequality, we have
    \[
        2(z-z') \leq (z+z') + \frac{(z-z')^2}{z+z'}.
    \]
    Rearranging then finishes the proof.
\end{proof}

\begin{proof}[Proof of Lemma~\ref{lem:dec_translation}]
Using \pref{lem:z'2z} with $z' = M_{i,j}^\star\paren{f(x,j) - v^\star_j}^2$ and $z = g(x,i,j)\paren{f(x,j) - v^\star_j}^2$, we have
\begin{align*}
&3M_{i,j}^\star\paren{f(x,j) - v^\star_j}^2
\geq g(x,i,j)\paren{f(x,j) - v^\star_j}^2 \\
&\qquad - \frac{(M_{i,j}^\star-g(x,i,j))^2}{M_{i,j}^\star+g(x,i,j)}\paren{f(x,j) - v^\star_j}^2 \\
&\geq g(x,i,j)\paren{f(x,j) - v^\star_j}^2  - \frac{(M_{i,j}^\star-g(x,i,j))^2}{M_{i,j}^\star+g(x,i,j)},
\end{align*}
where the last step uses the fact $\paren{f(x,j) - v^\star_j}^2 \leq 1$.
Applying this inequality to the definition of \pref{eqn:dec_graph_partial} and plugging in the parameters $\gamma_1 = \frac{3}{4}\gamma$ and $\gamma_2 = \frac{1}{4}\gamma$, we see that the two triangular discrimination terms cancel, leading us to exactly \pref{eqn:minimax_p_alg}.
\end{proof}
The importance of using log loss when learning the feedback graphs now becomes clear: the log loss regret turns out to be exactly the price one needs to pay by pretending that the graph estimator is the true graph.

The last step of the analysis is to show that the minimum DEC value, which our final strategy $p_t$ achieves, is reasonably small and related to some independence number with no polynomial dependence on the total number of actions:
\begin{restatable}{lemma}{decgamma}\label{lem:main-dec-lemma}
    For any $g \in \conv(\calG)$, $f\in\conv(\calF)$, $x\in \calX$, and $\gamma\geq 4$, we have
    \begin{align*}
\min_{p\in\Delta([K])}{\dec}_{\gamma}(p;f,g,x) = \order\left(\frac{\alpha(g,x)\log(K\gamma)}{\gamma}\right).
    \end{align*}
\end{restatable}
The proof of this lemma is deferred to \pref{app:partial_revealed} and is a refinement and generalization of \citet[Theorem 3.2]{zhang2023practical} which only concerns deterministic graphs.
Combining everything, we are now ready to prove \pref{thm:partial}.

\begin{proof}[Proof of Theorem~\ref{thm:partial}]
Setting $\gamma_1=\frac{3}{4}\gamma$ and $\gamma_2=\frac{1}{4}\gamma$ in \pref{thm:decbound} and combining it with \pref{lem:dec_translation},
we know that $\RegCB$ is at most
\begin{equation}\label{eqn:first_display}
\sum_{t=1}^T \dec_{\gamma}(p_t;f_t,g_t,x_t) + \frac{3}{4}\gamma \RegSq + \frac{1}{2}\gamma \RegLogG.
\end{equation}
Since $p_t$ is chosen by minimizing $\dec_{\gamma}(p;f_t,g_t,x_t)$,
we further apply \pref{lem:main-dec-lemma} to bound the regret by
\[
\order\left(\sum_{t=1}^T \frac{\alpha(g_t,x_t)\log(K\gamma)}{\gamma}\right) + \frac{3}{4}\gamma \RegSq + \frac{1}{2}\gamma \RegLogG.
\]
Finally, realizing $\alpha(g_t,x_t) \leq \alpha(\calG)$ (see \pref{lem:ind_trans}) and plugging in the choice of $\gamma$ finishes the proof.
\end{proof}

\subsection{Analysis for Fully Revealed Graphs}\label{sec:analysis_full}

From the analysis for the partially revealed graph setting, we see that the $\alpha(\calG)$ dependence in fact comes from $\alpha(g_t, x_t)$, the independence number with respect to the graph estimator $g_t$.
To improve it to $\alpha_t = \alpha(g^\star,x_t)$, we again need to connect two different graphs in the DEC definition using \pref{lem:z'2z}, as shown below.

\begin{restatable}{lemma}{decgammaFull}\label{lem:main-dec-lemma-full}
For $g, g'\in \conv(\calG)$, $f\in\conv(\calF)$, and $x\in \calX$, we have
\begin{align*}
&\min_{p\in\Delta([K])}\dec_{\gamma}(p;f,g',x) \leq \min_{p\in\Delta([K])} \dec_{\frac{\gamma}{3}}(p;f,g,x)  \\ &\qquad\qquad +\frac{\gamma}{12}\sum_{i=1}^K\sum_{j=1}^K\frac{(g(x,i,j)-g'(x,i,j))^2}{g(x,i,j)+g'(x,i,j)}.
\end{align*}
\end{restatable}
\begin{proof}[Proof sketch]
For a fix $p$, similarly to the proof of \pref{lem:dec_translation}, we apply \pref{lem:z'2z} with $z' = g'(x,i,j)(f(x,j)-v^\star_j)^2$ and $z = g(x,i,j)(f(x,j)-v^\star_j)^2$ for each $i$ and $j$ and arrive at
\begin{align*}
&\dec_{\gamma}(p;f,g',x) \leq \dec_{\frac{\gamma}{3}}(p;f,g,x)  \\ &\qquad\qquad +\frac{\gamma}{12}\sum_{i=1}^K p_i\sum_{j=1}^K\frac{(g(x,i,j)-g'(x,i,j))^2}{g(x,i,j)+g'(x,i,j)}.
\end{align*}
Further upper bounding $p_i$ by $1$ in the last term and then taking min over $p$ on both sides finishes the proof.
\end{proof}

This lemma allows us to connect the minimum DEC value with respect to $g_t$ and that with respect to $g^\star$, but with the price of $\frac{\gamma}{12}\sum_{i=1}^K \sum_{j=1}^K\frac{(g_t(x_t,i,j)-g^\star(x_t,i,j))^2}{g_t(x_t,i,j)+g^\star(x_t,i,j)}$, which, under fully revealed graphs, is essentially the per-round log loss regret in light of \pref{lem:log-sq-convert} since the oracle \AlgLog indeed receives observations for all $(i,j)$ pairs (importantly, this does not hold for partially revealed graphs). 
With this insight, we are ready to prove the main theorem.

\begin{proof}[Proof of Theorem~\ref{thm:fully}]
First note that \pref{thm:decbound} in fact also holds for fully revealed graphs; this is intuitively true simply because the fully revealed case is easier than the partially revealed case, and is formally explained in the proof of \pref{thm:decbound}.
Therefore, combining it with \pref{lem:dec_translation}, we still have $\RegCB$ bounded by
\[
\sum_{t=1}^T \min_p\dec_{\gamma}(p;f_t,g_t,x_t) + \frac{3}{4}\gamma \RegSq + \frac{1}{2}\gamma \RegLogG.
\]
By \pref{lem:main-dec-lemma-full}, this is at most
\begin{align*}
&\sum_{t=1}^T \min_p\dec_{\frac{\gamma}{3}}(p;f_t,g^\star,x_t) + \frac{3}{4}\gamma \RegSq + \frac{1}{2}\gamma \RegLogG  \\
&+ \frac{\gamma}{12}\sum_{t=1}^T\sum_{i=1}^K \sum_{j=1}^K\frac{(g_t(x_t,i,j)-g^\star(x_t,i,j))^2}{g_t(x_t,i,j)+g^\star(x_t,i,j)}.
\end{align*}
Finally, using \pref{lem:main-dec-lemma} and \pref{lem:log-sq-convert}, the above is further bounded by 
\[
\order\left(\sum_{t=1}^T \frac{\alpha_t\log(K\gamma)}{\gamma}\right) + \frac{3}{4}\gamma \RegSq + \frac{2}{3}\gamma \RegLogG,
\]
which completes the proof with our choice of $\gamma$.
\end{proof}

\section{Experiments}\label{sec:experiment}
In this section, we show empirical results of our \graphCBp algorithm  by testing it on a bidding application.
We start by describing this application, followed by modelling it as an instance of our partially revealed graph setting.
Specifically, consider a bidder (the learner) participating in a first-price auction.
At each round $t$, the bidder observes some context $x_t$, while the environment decides a competing price $w_t \in [0,1]$ (the highest price of all other bidders) and the value of the learner $v_t \in [0,1]$ for the current item (unknown to the learner herself).
Then, the learner decides her bid $c_t \in [0,1]$.
If $c_t \geq w_t$, the learner wins the auction, pays $c_t$ to the auctioneer (first-price), and observes her reward $v_t - c_t$;
otherwise, the learner loses the auction without observing her value $v_t$, and her reward is $0$.
In either case, at the end of this round, the auctioneer announces the winning bid to all bidders. For the learner, this information is only meaningful when she loses the auction, in which case $w_t$ (the winning bid) is revealed to her.

This problem is a natural instance of our model.
Specifically, we let the learner choose her bid $c_t$ from a discretized set $\calA_\epsilon=\{0,\epsilon,\dots,1-\epsilon,1\}$ of size $K=\frac{1}{\epsilon}+1$ for some granularity $\epsilon$. For ease of presentation, the $i$-th bid $(i-1)\epsilon$ is denoted by $a_i$.
The feedback graph $G_t$ is completely determined by the competing price $w_t$ in the following way:
\begin{align*}
G_{t,i,j}=\begin{cases} 1, & a_i < w_t \textrm{ and } a_j < w_t, \\
1, & a_i \ge w_t \textrm{ and } j \ge i, \\
0, & \textrm{otherwise},
\end{cases}
\end{align*}
where we again overload the notation $G_t$ to represent its adjacent matrix.
This is because when bidding lower than the competing price $w_t$, the learner observes $w_t$ and knows that bidding anything below $w_t$ gives $0$ reward; and when bidding higher than $w_t$, the bidder only knows that she would still win if she were to bid even higher, and the corresponding reward can be calculated since she knows her value $v_t$ in this case.
It is clear that this graph is strongly observable with independence number at most $2$ and is only partially revealed at the end of each round if the learner wins (and fully revealed otherwise).
On the other hand, the reward 
of action $i$ is $\mathbbm{1}\left[a_i \ge w_t\right] \cdot (v_t-a_i)$, which we translate to a loss in $[0,1]$ by shifting and scaling:
\begin{align}
\ell_{t,i} = \tfrac{1}{2}\left(1-\mathbbm{1}\left[a_i \ge w_t\right] \cdot (v_t-a_i)\right)\in[0,1]. \label{eqn:bidding_loss}
\end{align}

\paragraph{Regression oracles.} For the graph predictor, since the feedback graph is determined by the competing price $w_t$, we use a linear classification model to predict the distribution of $w_t$. Then at each round, we sample a competing price from this distribution, leading to the predicted feedback graph $g_t$. For the loss predictor $f_t$, since losses are determined by $w_t$ and $v_t$,
we use a two-layered fully connected neural network to predict the value $v_t$ and construct the loss predictors according to \pref{eqn:bidding_loss} with $w_t$ and $v_t$ replaced by their predicted values. For more details of the oracles and their training, see \pref{app:implementation}. 

\paragraph{Implementation of \graphCBp.} While \pref{eqn:minimax_p_alg} can be solved by a convex program, in order to implement \graphCBp even more efficiently, we use a closed-form solution of $p_t$ enabled by the specific structure of the predicted graph $g_t$ in this application. See \pref{app:implementation} for more details.

\subsection{Empirical Results on Synthetic Data}\label{sec:syn}
\paragraph{Data.} We first generate two synthetic datasets $\{(x_t,w_t,v_t)\}_{t=1}^T$  with $T=5000$ and $x_t\in \R^{32}$ for all $t\in[T]$. The competing price $w_t$ and the value $v_t$ are generated by $w_t=\frac{1}{\sqrt{32}}\theta_1^\top x_t+\epsilon_t, v_t=w_t+\max\{\frac{40}{\sqrt{32}}\theta_2^\top x_t, 0\}$, where $\theta_1, \theta_2\in \R^{32}$ are sampled from standard Gaussian and $\epsilon_t\sim \calN(0,0.05)$ is a small noise. All $w_t$ and $v_t$'s are then normalized to $[0,1]$. The two datasets only differ in how $\{x_t\}_{t=1}^T$ are generated. 
Specifically, in the context of linear bandits,
\citet{bastani2021mostly} showed that whether the simple greedy algorithm with no explicit exploration performs well depends largely on the context's diversity, roughly captured by the minimum eigenvalue of its covariance matrix.
We thus follow their work and generate two datasets where the first one enjoys good diversity and second one does not; see \pref{app:implementation} for details. 

\paragraph{Results.} We compare our \graphCBp with \squareCB \citep{foster2020beyond} (which ignores the additional feedback from graphs), the greedy algorithm (which simply picks the best action according to the loss predictors), and a trivial baseline that always bids $0$. 
For the first three algorithms, we try three different granularity values $\epsilon\in\{\frac{1}{25},\frac{1}{50},\frac{1}{75}\}$ leading to three increasing number of actions, and we run each of them 4 times and plot the averaged normalized regret ($\RegCB/T$) together with the standard deviation in \pref{fig:synthetic_dataset}. 
We observe that our algorithm performs the best and, unlike \squareCB, its regret almost does not increase when the number of action increases, matching our theoretical guarantee. In addition, consistent with~~\citet{bastani2021mostly}, while greedy indeed performs quite well when the contexts are diverse (top figure), it performs almost the same as the trivial ``not bidding'' baseline and suffers linear regret in the absence of diverse contexts (bottom figure).

\begin{figure}[!t]
\centering
\includegraphics[width=0.45\textwidth]{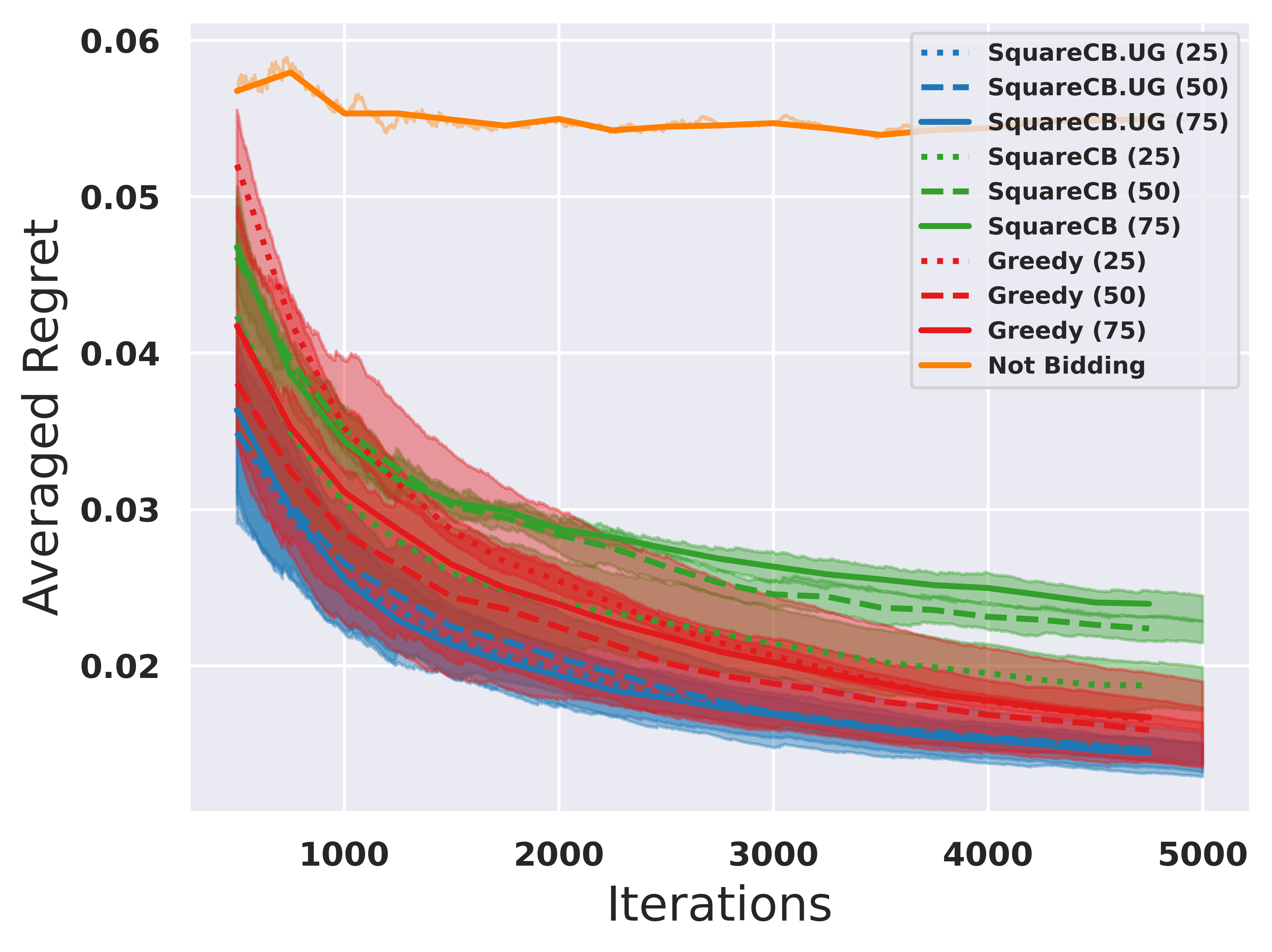}
\includegraphics[width=0.45\textwidth]{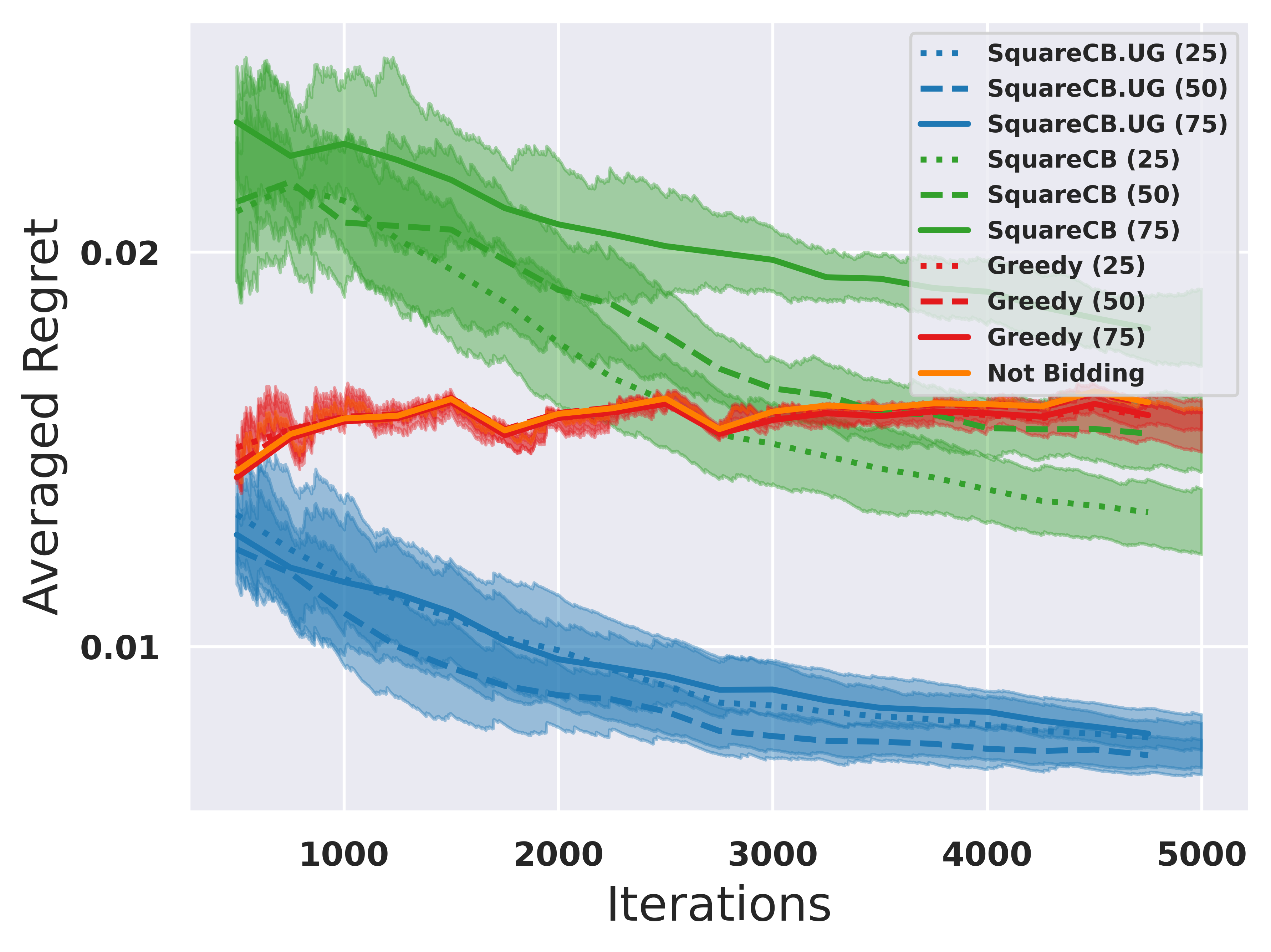}
\caption{
Comparison among \graphCBp, \squareCB, greedy, and a trivial baseline on one synthetic dataset with diverse contexts (top figure) and another one with poor diversity (bottom figure).}
\label{fig:synthetic_dataset}
\end{figure}

\begin{figure}[!t]
\centering
\includegraphics[width=0.45\textwidth]{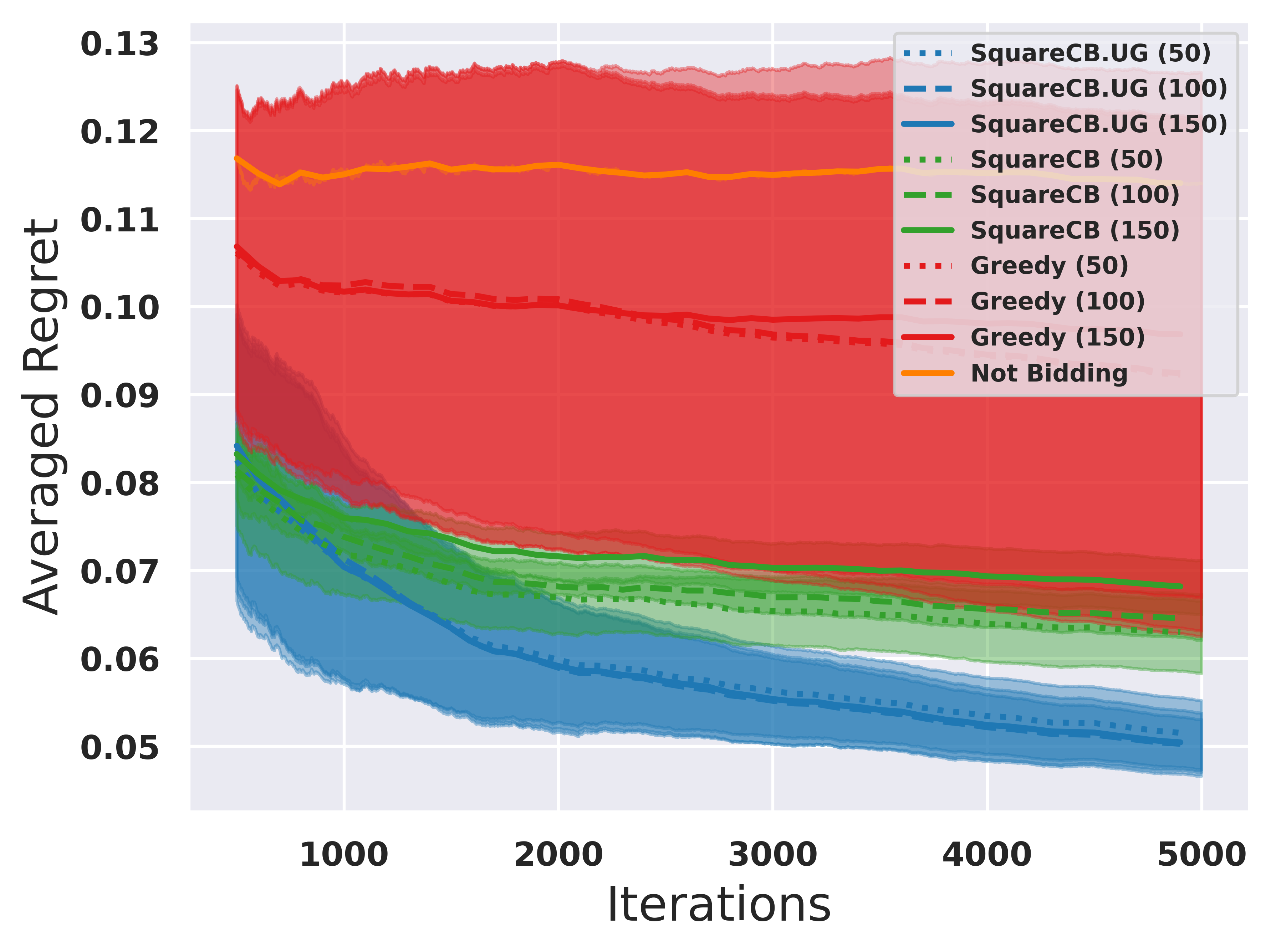}
\vspace{-3pt}
\caption{
Comparison among \graphCBp, \squareCB, greedy, and a trivial baseline on a real auction dataset. 
}
\label{fig:real_dataset}
\end{figure}

\subsection{Empirical Results on Real Auction Data}
\label{sec:real}
\paragraph{Data.}
We also conduct experiments on a subset of $5000$ samples of a real eBay auction dataset used in ~\citet{mohri2016learning};
see \pref{app:implementation} for details.

\paragraph{Results.}
We compare the four algorithms in the same way, with the only difference being the discretization granularity value
$\epsilon\in\{\frac{1}{50},\frac{1}{100},\frac{1}{150}\}$. The results are shown in \pref{fig:real_dataset}. Similar to what we observe in synthetic datasets, \graphCBp consistently outperforms other algorithms, demonstrating the advantage of exploration with graph information. The greedy algorithm's performance is unstable and has a relatively large variance 
due to the lack of exploration. 
With all these results for both real and synthetic datasets, we show that our algorithm indeed effectively explores the environment utilizing the uninformed graph structure and is robust to different types of environments.

%\paragraph{Broader Impact.} This paper presents work whose goal is to advance the field of Machine Learning. There are many potential societal consequences of our work, none of which we feel must be specifically highlighted here.

\bibliography{ref}
\bibliographystyle{icml2024}

\newpage
\appendix
\onecolumn
\section{Omitted Details in \pref{sec:alg}}\label{app:partial_revealed}
We start with restating \pref{thm:decbound} along with its proof.
\decbound*
\begin{proof}
Define $\pi^\star(x)=\argmin_{i\in[K]}f^\star(x,i)$. 
We decompose $\RegCB$ as follows:
\begin{align}
    &\RegCB \nonumber \\
    &= \E\left[ \sum_{t=1}^T\sum_{i=1}^Kp_{t,i}f^\star(x_t,i) - \sum_{t=1}^Tf^\star(x_t,\pi^\star(x_t)) \right] \nonumber \\
    &= \E\left[ \sum_{t=1}^T \left( \sum_{i=1}^Kp_{t,i}f^\star(x_t,i) - f^\star(x_t,\pi^\star(x_t)) - \gamma_1\sum_{i=1}^K\sum_{j=1}^Kp_{t,i}g^\star(x_t,i,j)\left(f^\star(x_t,j)-f_t(x_t,j)\right)^2\right)\right]  \nonumber \\
    & \qquad - \E\left[\sum_{t=1}^T\gamma_2\sum_{i=1}^K\sum_{j=1}^Kp_{t,i}\frac{(g^\star(x_t,i,j)-g_t(x_t,i,j))^2}{g^\star(x_t,i,j)+g_t(x_t,i,j)}\right]+\gamma_1\E\left[\sum_{t=1}^T\sum_{i=1}^K\sum_{j=1}^Kp_{t,i}g^\star(x_t,i,j)\left(f^\star(x_t,j)-f_t(x_t,j)\right)^2\right] \nonumber \\
    &\qquad + \gamma_2\E\left[\sum_{t=1}^T\sum_{i=1}^K\sum_{j=1}^Kp_{t,i}\frac{(g^\star(x_t,i,j)-g_t(x_t,i,j))^2}{g^\star(x_t,i,j)+g_t(x_t,i,j)}\right] \nonumber \\
    &\leq \E\left[ \sum_{t=1}^T \max_{\substack{i^\star \in [K] \\ v^\star\in[0,1]^K, M^\star\in [0,1]^{K\times K} }}\left( \sum_{i=1}^Kp_{t,i}v_i^\star - v^\star_{i^\star} - \gamma_1 \sum_{i=1}^K\sum_{j=1}^Kp_{t,i}M^\star_{i,j}\left(v^\star_j-f_t(x_t,j)\right)^2\right.\right.\nonumber \\
    & \qquad \left.\left.-\gamma_2\sum_{i=1}^K\sum_{j=1}^Kp_{t,i}\frac{(M_{i,j}^\star-g_t(x_t,i,j))^2}{M_{i,j}^\star+g_t(x_t,i,j)}\right)\right]+\gamma_1\E\left[\sum_{t=1}^T\sum_{i=1}^K\sum_{j=1}^Kp_{t,i}g^\star(x_t,i,j)\left(f^\star(x_t,j)-f_t(x_t,j)\right)^2\right] \nonumber \\
    &\qquad+ \gamma_2\E\left[\sum_{t=1}^T\sum_{i=1}^K\sum_{j=1}^Kp_{t,i}\frac{(g^\star(x_t,i,j)-g_t(x_t,i,j))^2}{g^\star(x_t,i,j)+g_t(x_t,i,j)}\right] \nonumber \\
    & = \E\left[\sum_{t=1}^T\decp_{\gamma_1, \gamma_2}(p_t,f_t,g_t,x_t)\right] +\gamma_1\E\left[\sum_{t=1}^T\sum_{i=1}^K\sum_{j=1}^Kp_{t,i}g^\star(x_t,i,j)\left(f^\star(x_t,j)-f_t(x_t,j)\right)^2\right] \nonumber \\
    &\qquad + \gamma_2\E\left[\sum_{t=1}^T\sum_{i=1}^K\sum_{j=1}^Kp_{t,i}\frac{(g^\star(x_t,i,j)-g_t(x_t,i,j))^2}{g^\star(x_t,i,j)+g_t(x_t,i,j)}\right].\label{eqn:dec_reg_inter}
\end{align}
Next, since $\E[\ell_{t,a}\;\vert\; x_t]=f^\star(x_t,a)$ for all $t\in [T]$ and $a\in [K]$, we know that
\begin{align}
    &\E\left[ \sum_{t=1}^T\sum_{i=1}^K\sum_{j=1}^Kp_{t,i}g^\star(x_t,i,j)\left(f^\star(x_t,j)-f_t(x_t,j)\right)^2 \right] \nonumber \\
    &= \E\left[\sum_{t=1}^T\E_{A_t}\left[\sum_{i\in A_t}(f_t(x_t,i)-\ell_{t,i})^2 - \sum_{i\in A_t}(f^\star(x_t,i)-\ell_{t,i})^2 \right]\right]\leq \RegSq, \label{eqn:regressor_reg}
\end{align}
where the inequality is due to \pref{asm:regression_oracle_loss} and the way the algorithm feeds the oracle \AlgSq. In addition, according to \pref{lem:log-sq-convert} and the way the algorithm feeds the oracle \AlgLog in the partially revealed graph setting, we know that
\begin{align}\label{eqn:regressor_reg_graph}
    \E\left[\sum_{t=1}^T\sum_{i=1}^K\sum_{j=1}^Kp_{t,i}\frac{(g^\star(x_t,i,j)-g_t(x_t,i,j))^2}{g^\star(x_t,i,j)+g_t(x_t,i,j)}\right] = \E\left[\sum_{t=1}^T\E_{i_t\sim p_t}\left[\sum_{i=1}^K\frac{(g^\star(x_t,i_t,i)-g_t(x_t,i_t,i))^2}{g^\star(x_t,i_t,i)+g_t(x_t,i_t,i)}\right]\right] \leq 2\RegLogG.
\end{align}
Combining the above two inequalities, we obtain that
\begin{align}\label{eqn:regcb_regsq}
\RegCB \leq \E\left[\sum_{t=1}^T \decp_{\gamma_1,\gamma_2}(p_t;f_t,g_t,x_t)\right] + \gamma_1 \RegSq + 2\gamma_2 \RegLogG.
\end{align}
Moreover, \pref{eqn:regcb_regsq} also holds in the fully revealed feedback graph setting since in this case, according to \pref{lem:log-sq-convert} and the fact that the algorithm feeds all action-pairs to \AlgLog, we have
\begin{align}\label{eqn:fully_obs}
\E\left[\sum_{t=1}^T\sum_{i=1}^K\sum_{j=1}^Kp_{t,i}\frac{(g^\star(x_t,i,j)-g_t(x_t,i,j))^2}{g^\star(x_t,i,j)+g_t(x_t,i,j)}\right] \leq \E\left[\sum_{t=1}^T\left[\sum_{i=1}^K\sum_{j=1}^K\frac{(g^\star(x_t,i,j)-g_t(x_t,i,j))^2}{g^\star(x_t,i,j)+g_t(x_t,i,j)}\right]\right] \leq 2\RegLogG.
\end{align}
Plugging \pref{eqn:fully_obs} and \pref{eqn:regressor_reg} into \pref{eqn:dec_reg_inter} finishes the proof in the fully revealed feedback graph setting.

\end{proof}

\subsection{Value of the Minimax Program}
In this subsection, we show that for any context $x\in\calX$, $g\in \conv(\calG)$ and $f\in\conv(\calF)$, the minimum DEC value is roughly of order $\frac{\alpha(g,x)}{\gamma}$.

\decgamma*

\begin{proof}
    To bound $\dec_\gamma(p;f,g,x)$, it suffices to bound $\overline{\dec}_\gamma(p;f,g,x)$ defined as follows, which relaxes the constraint from $v^\star\in[0,1]^K$ to $v^\star\in \R^K$:
    \begin{align*}
        \overline{\dec}_\gamma(p;f,g,x) = \max_{\substack{i^\star \in [K] \nonumber \\ v^\star\in \R^K }}
	\left[ \sum_{i=1}^Kp_iv^\star_i - v^\star_{i^\star} - \frac{1}{4}\gamma\sum_{i=1}^K\sum_{j=1}^Kp_jg(x,i,j)\paren{f(x,i) - v^\star_i}^2  \right].
    \end{align*}
    %Define $W(p,g)\in[0,1]^{K\times K}$ be the diagonal matrix with the $i$-th diagonal entry $W(p,g)_{i,i}=\sum_{j=1}^Kp_{j}g(j,i)$
    For a positive definite matrix $M\in \R^{K\times K}$, we define norm $\|z\|_M=\sqrt{z^\top Mz}$. By taking the gradient with respect to $v^\star$ and setting it to zero, we know that for any $i^\star\in[K]$ and $p\in\Delta([K])$,
    \begin{align}\label{eqn:gradient}
        &\max_{v^\star\in\R^K}\left\{\sum_{i=1}^Kp_iv^\star_i- v^\star_{i^\star} - \frac{1}{4}\gamma\sum_{i=1}^K\sum_{j=1}^Kp_jg(x,i,j)(f(x,i)-v^\star_i)^2\right\} \nonumber\\
        &=\sum_{i=1}^Kp_if(x,i)-f(x,i^\star) + \frac{1}{\gamma}\|p-e_{i^\star}\|^2_{W(p,g,x)},
    \end{align}
    where $W(g,p,x)$ is a diagonal matrix with the $i$-th diagonal entry being $\sum_{j=1}^Kp_jg(x,j,i)$ and $e_{i^\star}\in \R^K$ corresponds to the basic vector with the $i$-th coordinate being $1$.
    
    Then, direct calculation shows that
    \begin{align}
    &\min_{p\in\Delta([K])}\overline{\dec}_{\gamma}(p;f,g,x) \nonumber \\
    &= \min_{p\in \Delta([K])}\max_{i^\star\in [K]}\max_{v^\star\in\R^K}\left\{\sum_{i=1}^Kp_iv^\star_i - v^\star_{i^\star} - \frac{1}{4}\gamma\sum_{i=1}^K\sum_{j=1}^Kp_jg(x,i,j)(f(x,i)-v^\star_i)^2\right\}\nonumber\\
    &= \min_{p\in \Delta([K])}\max_{i^\star\in [K]}\left\{\sum_{i=1}^Kp_{i}f(x,i) - f(x,i^\star) +\frac{1}{\gamma}\|p-e_{i^\star}\|^2_{W(p,g,x)}\right\} \tag{according to \pref{eqn:gradient}}\\
    &= \min_{p\in \Delta([K])}\max_{i^\star\in [K]}\left\{\sum_{i=1}^Kp_{i}f(x,i) - f(x,i^\star) +\frac{1}{\gamma}\sum_{i\neq i^\star}\frac{p_i^2}{\sum_{i'=1}^Kp_{i'}g(x,i',i)} + \frac{1}{\gamma}\frac{(1-p_{i^\star})^2}{\sum_{i'=1}^Kp_{i'}g(x,i',i^\star)}\right\} \label{eqn:exxx}\\
    &= \min_{p\in \Delta([K])}\max_{q\in \Delta([K])}\left\{\sum_{i=1}^Kp_{i}f(x,i) - \sum_{i=1}^Kq_if(x,i) +\frac{1}{\gamma}\sum_{i=1}^K\frac{(1-q_i)p_i^2}{\sum_{i'=1}^Kp_{i'}g(x,i',i)} + \frac{1}{\gamma}\sum_{i=1}^K\frac{q_i(1-p_{i})^2}{\sum_{i'=1}^Kp_{i'}g(x,i',i)}\right\}\nonumber\\
    &= \max_{q\in \Delta([K])}\min_{p\in \Delta([K])}\left\{\sum_{i=1}^Kp_{i}f_t(x,i)  - \sum_{i=1}^Kq_if(x,i) +\frac{1}{\gamma}\sum_{i=1}^K\frac{(1-q_i)p_i^2}{\sum_{i'=1}^Kp_{i'}g(x,i',i)} + \frac{1}{\gamma}\sum_{i=1}^K\frac{q_i(1-p_{i})^2}{\sum_{i'=1}^Kp_{i'}g(x,i',i)}\right\} \label{eqn:exx},
\end{align}
where the last inequality is due to Sion's minimax theorem.

Picking $p_a=\left(1-\frac{1}{\gamma}\right)q_a + \frac{1}{\gamma K}$ for all $a\in[K]$, we  obtain that for any distribution $\mu\in \calQ(g,x)$,
\begin{align}
&\max_{q\in \Delta([K])}\min_{p\in \Delta([K])}\left\{\sum_{i=1}^Kp_{i}f(x,i) - \sum_{i=1}^Kq_if(x,i) +\frac{1}{\gamma}\sum_{i=1}^K\frac{(1-q_i)p_i^2}{\sum_{i'=1}^Kp_{i'}g(x,i',i)} + \frac{1}{\gamma}\sum_{i=1}^K\frac{q_i(1-p_{i})^2}{\sum_{i'
    \in \calA}p_{i'}g(x,i',i)}\right\} \nonumber\\
    &\stackrel{(i)}{\leq} \max_{q\in \Delta([K])}\left\{-\frac{1}{\gamma}\sum_{i=1}^Kq_{i}f(x,i) + \frac{1}{\gamma K}\sum_{i=1}^Kf(x,i) +\frac{1}{\gamma}\sum_{i=1}^K\frac{(1-q_i)\left(q_i-\frac{1}{\gamma}q_i+\frac{1}{\gamma K}\right)^2+q_i\left(1-(1-\frac{1}{\gamma})q_i-\frac{1}{\gamma K}\right)^2}{\sum_{i'=1}^Kp_{i'}g(x,i',i)} \right\} \nonumber\\
    &\stackrel{(ii)}{\leq}\max_{q\in\Delta([K])} \left\{\frac{1}{\gamma} + \frac{1}{\gamma}\sum_{i=1}^{ K}\frac{2\left((1-\frac{1}{\gamma})^2q_i^2+\frac{1}{\gamma^2 K^2}\right)(1-q_i) + q_i\left(1-(1-\frac{1}{\gamma})q_i\right)^2}{\sum_{i'=1}^Kp_{i'}g_t(x_t,i',i)}\right\} \nonumber\\
    &\stackrel{}{\leq}\max_{q\in\Delta([K])} \left\{\frac{1}{\gamma} + \frac{1}{\gamma}\sum_{i=1}^{ K}\frac{2\left(q_i^2+\frac{1}{\gamma^2 K^2}\right)(1-q_i) + q_i\left((1-q_i)+\frac{q_i}{\gamma}\right)^2}{\sum_{i'=1}^Kp_{i'}g_t(x_t,i',i)}\right\} \nonumber\\
    &\stackrel{(iii)}{\leq} \max_{q\in\Delta([K])} \left\{\frac{1}{\gamma} + \frac{2}{\gamma^2} + \frac{1}{\gamma}\sum_{i=1}^{K}\frac{2q_i^2(1-q_i) + 2q_i\left(1-q_i\right)^2 + \frac{2q_i^3}{\gamma^2}}{\sum_{i'=1}^Kp_{i'}g(x,i',i)}\right\} \nonumber\\%\tag{$\sum_{i'=1}^Kp_{i'}g_t(x_t,i',i)\geq \frac{1}{\gamma K}$ for all $i\in[K]$}\\
    &= \max_{q\in\Delta([K])} \left\{\frac{1}{\gamma} + \frac{2}{\gamma^2} + \frac{2}{\gamma}\sum_{i=1}^{ K}\frac{q_i(1-q_i)}{\sum_{i'=1}^Kp_{i'}g(x,i',i)} + \frac{2}{\gamma^3}\sum_{i=1}^K\frac{q_i^3}{\sum_{i'=1}^Kp_{i'}g(x,i',i)}\right\} \nonumber\\
    &\stackrel{(iv)}{=} \max_{q\in\Delta([K])} \left\{\frac{1}{\gamma} + \frac{2}{\gamma^2} + \frac{2}{\gamma}\sum_{i=1}^{K}\frac{q_i(1-q_i)}{\E_{G\sim \mu}\left[\sum_{i'=1}^Kp_{i'}G_{i',i}\right]} + \frac{2}{\gamma^3}\sum_{i=1}^K\frac{q_i^3}{\E_{G\sim \mu}\left[\sum_{i'=1}^Kp_{i'}G_{i',i}\right]}\right\}\nonumber,
    \end{align}
where $(i)$ is by replacing $p_i$ with $(1-\frac{1}{\gamma})q_i+\frac{1}{\gamma K}$ (except for the $p_{i'}$ in the denominators); $(ii)$ holds since $f(x,i)\in[0,1]$ for all $i\in[K]$ and $(a+b)^2\leq 2(a^2+b^2)$, and we drop the last $-\frac{1}{\gamma K}$ term; $(iii)$ is because $\sum_{i'=1}^Kp_{i'}g(x,i',i)\geq \frac{1}{\gamma K}$ for all $i\in[K]$ since $p_{i'}\geq \frac{1}{\gamma K}$ for all $i'\in[K]$ and $\sum_{i'=1}^K g(x, i', i) \geq 1$ as $g(x,\cdot,\cdot)$ is the mean graph of a distribution of strongly observable graphs; $(iv)$ is by definition of $\calQ(g,x)$ and with an abuse of notation, $G_{i,j}$ represents the $(i,j)$-th entry of the adjacent matrix of $G$.

For a feedback graph $G\in\{0,1\}^{K\times K}$ and a distribution $u\in\Delta([K])$, with an abuse of notation, define $W(G,u)\in [0,1]^K$ as the probability for each node $i\in[K]$ to be observed according to $u$ and $G$. Specifically, for each $i\in[K]$, $W_i(G,u)=\sum_{j=1}^Ku_jG_{j,i}$. In addition, let $S\subseteq[K]$ be the nodes in $G$ that have self-loops, meaning that $G_{i,i}=1$ for all $i\in S$. Then, using Jensen's inequality, we know that
\begin{align}
    &\min_{p\in\Delta([K])}\overline{\dec}_{\gamma}(p;f,g,x) \nonumber \\
    &\leq \max_{q\in\Delta([K])} \E_{G\sim \mu}\left[\frac{1}{\gamma} + \frac{2}{\gamma^2} + \frac{2}{\gamma}\sum_{i=1}^{K}\frac{q_i(1-q_i)}{\sum_{i'=1}^Kp_{i'}G_{i',i}} + \frac{2}{\gamma^3}\sum_{i=1}^K\frac{q_i^3}{\sum_{i'=1}^Kp_{i'}G_{i',i}}\right] \tag{Jensen's inequality} \\
    &\stackrel{(i)}{=} \max_{q\in\Delta([K])} \E_{G\sim \mu}\left[\frac{1}{\gamma} + \frac{2}{\gamma^2} + \frac{2}{\gamma}\sum_{i=1}^{K}\frac{q_i(1-q_i)}{W_{i}(G,p)} + \frac{2}{\gamma^3}\sum_{i\in S}\frac{q_i^3}{W_{i}(G,p)} +\frac{2}{\gamma^3}\sum_{i\notin S}\frac{q_i^3}{W_i(G,p)} \right] \nonumber \\
    &\stackrel{(ii)}{\leq}\max_{q\in\Delta([K])}\E_{G\sim \mu} \left[\frac{1}{\gamma} + \frac{2}{\gamma^2} + \frac{2}{\gamma}\sum_{i=1}^{ K}\frac{q_i(1-q_i)}{W_i(G,p)} + \frac{2}{\gamma^2(\gamma-1)}\sum_{i\in S}q_i^2+\frac{2}{\gamma^3}\sum_{i\notin S}\frac{q_i^3}{\frac{K-1}{\gamma K}} \right]\nonumber\\%\tag{if $a\notin S$, $W_a=1-p_a\geq \frac{K-1}{\gamma K}$}\nonumber\\
    &\stackrel{(iii)}{\leq}\max_{q\in\Delta([K])} \E_{G\sim \mu}\left[\order\left(\frac{1}{\gamma}+ \frac{1}{\gamma}\sum_{i=1}^{K}\frac{q_i(1-q_i)}{W_i(G,p)}\right)\right] \nonumber\\
    &\stackrel{(iv)}{\leq} \order\left(\frac{1}{\gamma}\right)+ \order\left(\frac{1}{\gamma}\right)\cdot\E_{G\sim \mu}\left[\max_{q\in\Delta([K])} \sum_{i=1}^{K}\frac{q_i(1-q_i)}{W_i(G,p)}\right] \label{eqn:dec_final_1},
\end{align}
where $(i)$ is by definition of $W(G,p)$; $(ii)$ is because for all $i \in S$, $W_i(G,p) \geq p_i \geq (1-\frac{1}{\gamma})q_i$, and for all $i\notin S$, every other node in $G$ can observe $i$ and $W_i(G,p)=1-p_i\geq \frac{K-1}{\gamma K}$ since $p_i\geq \frac{1}{\gamma K}$; $(iii)$ is because $K\geq 2$ and $\gamma\geq 4$; $(iv)$ is again due to Jensen's inequality. 

Next we bound $\sum_{i=1}^K\frac{q_i(1-q_i)}{W_i(G,p)}$ for any strongly observable graph $G$ with independence number $\alpha$. For notational convenience, we omit the index $G$ and denote $W_i(G,p)$ by $W_i(p)$. If $i\in [K]\backslash S$, we know that $W_i(p)=1-p_i$ and
\begin{align}\label{eqn:1}
    \frac{q_i(1-q_i)}{W_i(p)} = \frac{q_i(1-q_i)}{1-\left(1-\frac{1}{\gamma}\right)q_i-\frac{1}{\gamma K}} = \frac{q_i(1-q_i)}{\left(1-\frac{1}{\gamma}\right)(1-q_i)+\frac{K-1}{\gamma K}} \leq \frac{1}{1-\frac{1}{\gamma}}q_i\leq 2q_i,
\end{align}
where the first equality is because $p_i=(1-\frac{1}{\gamma})q_i+\frac{1}{\gamma K}$ for all $i\in[K]$ and the last inequality is because $\gamma\geq 4$.
If $i\in S$, we know that
\begin{align}\label{eqn:2}
    \sum_{i\in S}\frac{q_i(1-q_i)}{W_i(p)} &= \sum_{i\in S}\frac{q_i(1-q_i)}{\sum_{j:G_{j,i}=1}\left(\left(1-\frac{1}{\gamma}\right)q_j+\frac{1}{\gamma K}\right)} \nonumber\\
    &\leq \frac{\gamma}{\gamma-1}\sum_{i\in S}\frac{\left((1-\frac{1}{\gamma})q_i+\frac{1}{\gamma K}\right)(1-q_i)}{\sum_{j:G_{j,i}=1}\left(\left(1-\frac{1}{\gamma}\right)q_j+\frac{1}{\gamma K}\right)}\nonumber\\
    &\leq 2\sum_{i\in S}\frac{\left((1-\frac{1}{\gamma})q_i+\frac{1}{\gamma K}\right)(1-q_i)}{\sum_{j:G_{j,i}=1}\left(\left(1-\frac{1}{\gamma}\right)q_j+\frac{1}{\gamma K}\right)} \nonumber\\
    &\leq \order(\alpha\log(K\gamma)),
\end{align}
where the last inequality is due to \pref{lem:alon}. 

Combining \pref{eqn:1} and \pref{eqn:2}, we know that for any $q\in\Delta([K])$ and strongly observable graph $G$ with independence number $\alpha$,
    $\sum_{i=1}^K\frac{q_i(1-q_i)}{W_i(G,p)}\leq \order(\alpha\log(K\gamma))$. Plugging the above back to \pref{eqn:dec_final_1}, we know that 
    \begin{align*}
        \min_{p\in\Delta([K])}\overline{\dec}_{\gamma}(p;f,g,x)\leq \order\left(\frac{1}{\gamma}\right)\cdot\E_{G\sim \mu}\left[\order(\alpha(G)\log(K\gamma))\right] \leq \order\left(\frac{\alpha(g,x)\log(K\gamma)}{\gamma}\right),
    \end{align*}
    where the last inequality is due to the definition of $\alpha(g,x)$.

\end{proof}

The following two auxiliary lemmas have been used in our analysis.
\begin{restatable}{lemma}{ind_trans}\label{lem:ind_trans}
For all $\wh{g}\in\conv(\calG)$ and context $x\in\calX$, $\alpha(\wh{g},x)\leq \alpha(\calG)$.
\end{restatable}
\begin{proof}
Let $u\in\Delta(\calG)$ be such that $\E_{g\sim u}[g]=\wh{g}$.
For each $g \in \calG$, consider any $q_g \in \calQ(g,x)$. We have by definition $\wh{q} \triangleq \E_{g\sim u}[q_g]\in \calQ(\wh{g},{x})$, leading to
\[
\alpha(\wh{g},x)=\inf_{q\in\calQ(\wh{g},{x})}\E_{G\sim q}[\alpha(G)] \leq \E_{G\sim \wh{q}}[\alpha(G)] = \E_{g\sim u}\left[\E_{G\sim q_g}[\alpha(G)] \right].
\]
Since $q_g$ can be any distribution in $\calQ(g,x)$,
the above implies
\[
\alpha(\wh{g},x) \leq \E_{g\sim u}\left[\inf_{\rho \in \calQ(g,x)}\E_{G\sim \rho}[\alpha(G)] \right]
= \E_{g\sim u}\left[\alpha(g, x)\right],
\]
which is at most $\E_{g\sim u}[\sup_{x\in\calX}\alpha(g,x)] \leq \alpha(\calG)$, finishing the proof.
\end{proof}

\begin{lemma}[Lemma 5 in~\citep{alon2015online}]\label{lem:alon}
Let $G = (V, E)$ be a directed graph with $|V|=K$, in which $G_{i,i}=1$ for all vertices $i\in [K]$. Assign each $i\in V$ with a positive weight $w_i$ such that $\sum_{i=1}^n w_i\leq 1$ and $w_i\geq \epsilon$ for all $i\in V$
for some constant $0 < \epsilon< \frac{1}{2}$. Then
\begin{align*}
    \sum_{i=1}^K\frac{w_i}{\sum_{j: G_{j,i}=1}w_j} \leq 4\alpha(G) \ln\frac{4K}{\alpha(G)\epsilon},
\end{align*}
where $\alpha(G)$ is the independence number of $G$.
\end{lemma}

\subsection{Parameter-Free Algorithm in the Partially Revealed Feedback Graphs Setting}\label{app:parameter_free}

In this section, we show that applying doubling trick to \pref{alg:squareCB.GPLUS} achieves the same regret without the knowledge of $\alpha(\calG)$ in the partially revealed feedback setting. The idea follows~\citet{zhang2023practical}, which utilizes the value of the minimax problem \pref{eqn:minimax_p_alg} to guide the choice of $\gamma$.

Specifically, our algorithm goes in epochs with the parameter $\gamma$ being $\gamma_{s}$ in the $s$-th epoch and $\gamma_1=\sqrt{\frac{T}{\max\{\RegSq,\RegLogG\}}}$.
Within each epoch $s$ (with starting round $b_s$), at round $t$, we calculate the value
\begin{align}\label{eqn:doubling}
    R_t=\sum_{\tau=b_s}^t \min_{p\in\Delta([K])}{\dec}_{\gamma_{s}}(p;{f}_\tau,g_\tau,x_\tau),
\end{align}
and decide whether to start a new epoch by checking whether $R_t \leq \gamma_{s}\max\left\{\RegSq,\RegLogG\right\}$. Specifically, if $R_t \leq \gamma_{s}\max\left\{\RegSq,\RegLogG\right\}$, we continue our algorithm using $\gamma_{s}$; otherwise, we set $\gamma_{s+1}=2\gamma_s$ and restart the algorithm. 

Now we analyze the performance of the above described algorithm. Denote the $s$-th epoch to be $I_s=\{b_s,b_s+1, \ldots, e_s\} \triangleq [b_s, e_s]$ and let $S$ be the total number of epochs. First, using \pref{lem:ind_trans} and \pref{lem:main-dec-lemma}, we know that for any $t$ within epoch $s$, we have
\begin{align*}
    \gamma_{s} R_t\leq \otil\left(\sum_{t\in I_s}\alpha(\calG)\right)\leq \otil(\alpha(\calG)T).
\end{align*}
Applying this to the last round of the $(S-1)$-th epoch, we obtain:
\[
\gamma_{S-1}^2\max\left\{\RegSq,\RegLogG\right\} \leq \otil(\alpha(\calG)T),
\]
which, together with $\gamma_{S-1} = 2^{S-2}\gamma_1$ and the definition of $\gamma_1$, implies $2^S = \otil(\sqrt{\alpha(G)})$.

Next, consider the regret in epoch $s$. According to \pref{eqn:first_display}, we know that the regret within epoch $s$ is bounded as follows:
\begin{align}
    &\E\left[\sum_{t\in I_s}f^\star(x_t,i_t)-\sum_{t\in I_s}\min_{i\in[K]}f^\star(x_t,i)\right] \nonumber\\
    &\leq \E\left[\sum_{t\in I_s}\dec_{\gamma_s}(p_t;f_t,g_t,x_t)\right] + \frac{3\gamma_s}{4}\RegSq + \frac{1}{2}\gamma_s\RegLogG\nonumber\\
    &\leq \E\left[\sum_{t\in [b_s,e_s-1]}\dec_{\gamma_s}(p_t;f_t,g_t,x_t)\right] + \otil\left(\frac{\alpha(\calG)}{\gamma_s}\right) + 2\gamma_s\max\{\RegSq, \RegLogG\} \nonumber\\
    &\leq \otil(\gamma_s\max\{\RegSq, \RegLogG\}),\label{eqn:double-1}
\end{align}
where the second inequality uses \pref{lem:dec_translation} and \pref{lem:main-dec-lemma} again, and the last inequality is because at round $t=e_s-1$, $R_t\leq \gamma_s\max\{\RegSq,\RegLogG\}$ is satisfied and $\frac{\alpha(\calG)}{\gamma_s} \leq \frac{T}{\gamma_s} \leq \gamma_s \max\{\RegSq,\RegLogG\}$. Taking summation over all $S$ epochs, we know that the overall regret is bounded as
\begin{align}
    \E[\RegCB]&\leq \sum_{s=1}^{S}\otil(\gamma_s\max\{\RegSq, \RegLogG\}) = \sum_{s=1}^S\otil\left(2^{s-1}\sqrt{T\max\{\RegSq, \RegLogG\}}\right) \nonumber\\
    &\leq \otil\left(2^S\sqrt{T\max\{\RegSq, \RegLogG\}}\right) = \otil\left(\sqrt{\alpha(\calG)T\max\{\RegSq,\RegLogG\}}\right)\label{eqn:double-2},
\end{align}
which is exactly the same as \pref{thm:partial}.
\section{Omitted Details in \pref{sec:analysis_full}}

\decgammaFull*
\begin{proof}
    For any $v\in[0,1]^K$, applying \pref{lem:z'2z} with $z' = g'(x,i,j)(f(x,j)-v_j)^2$ and $z = g(x,i,j)(f(x,j)-v_j)^2$ for each $i$ and $j$, we know that 
    \begin{align}
&-3\gamma\sum_{i=1}^K\sum_{j=1}^Kp_ig'(x,i,j)(f(x,j)-v_j)^2 - \gamma\sum_{i=1}^K\sum_{j=1}^Kp_i\frac{(g(x,i,j) - g'(x,i,j))^2}{g(x,i,j)+g'(x,i,j)} \nonumber\\
    &\leq -\gamma\sum_{i=1}^K\sum_{j=1}^Kp_ig(x,i,j)(f(x,j)-v_j)^2 \label{eqn:translation_appendix_sym}.
\end{align}
Plugging this back to the definition of $\dec_{\gamma}(p;f,g,x)$, we get
\begin{align*}
    &\min_{p\in\Delta([K])}{\dec}_{\gamma}(p;f,g',x)\\
    &=\min_{p\in\Delta([K])} \max_{\substack{i^\star \in [K] \\ v^\star\in[0,1]^K}}
	\left[ \sum_{i=1}^Kp_iv^\star_i - v^\star_{i^\star} -\frac{1}{4}\gamma\sum_{i=1}^K\sum_{j=1}^Kp_ig'(x,i,j)\paren{f(x,j) - v^\star_j}^2 \right]\\
    &\stackrel{(i)}{\leq} \min_{p\in\Delta([K])} \max_{\substack{i^\star \in [K] \\ v^\star\in[0,1]^K}}
	\left[ \sum_{i=1}^Kp_iv^\star_i - v^\star_{i^\star} -\frac{1}{12}\gamma\sum_{i=1}^K\sum_{j=1}^Kp_ig(x,i,j)\paren{f(x,j) - v^\star_j}^2 \right.\\
    &\qquad\left.+\frac{1}{12}\gamma\sum_{i=1}^K\sum_{j=1}^Kp_i\frac{(g(x,i,j)-g'(x,i,j))^2}{g(x_t,i,j)+g'(x,i,j)}\right]\\
    &\stackrel{(ii)}{\leq} \min_{p\in\Delta([K])} \max_{\substack{i^\star \in [K] \\ v^\star\in [0,1]^K}}
	\left[ \sum_{i=1}^Kp_iv^\star_i - v^\star_{i^\star} -\frac{1}{12}\gamma\sum_{i=1}^K\sum_{j=1}^Kp_ig(x,i,j)\paren{f(x,j) - v^\star_j}^2 \right]\\
    &\qquad+ \frac{1}{12}\gamma\sum_{i=1}^K\sum_{j=1}^K \frac{(g(x,i,j)-g'(x,i,j))^2}{g(x,i,j)+g'(x,i,j)} \\
    &= \min_{p\in\Delta([K])}{\dec}_{\frac{\gamma}{3}}(p;f,g,x)+\frac{1}{12}\gamma\sum_{i=1}^K\sum_{j=1}^K\frac{(g(x,i,j)-g'(x,i,j))^2}{g(x,i,j)+g'(x,i,j)},
\end{align*}
where $(i)$ uses \pref{eqn:translation_appendix_sym} and $(ii)$ holds by trivially bounding $p_i$ in the last term by $1$.
\end{proof}

\section{Implementation Details}
\label{app:implementation}

We first point out that the DEC defined in~\citet{zhang2023practical} in fact relaxes the constraint $v^\star\in[0,1]^K$ to $v^\star\in \mathbb{R}^K$, which makes the problem of minimizing the DEC a simple convex program (see their Theorem 3.6).
In the partially revealed graph setting, we in fact can do the exact same trick because it does not affect our analysis at all.
Since our experiments are for an application with partially revealed graphs, we indeed implemented our algorithm in this way for simplicity.

However, this relaxation does not work for the fully revealed graph setting, since the analysis of \pref{lem:main-dec-lemma-full} relies on $v^\star \in[0,1]^K$.
Nevertheless, minimizing the DEC is still a relatively simple convex problem.
To see this, we first fix an $i^\star \in [K]$ and work on the supremum over $v^\star \in [0,1]^K$.
Specifically, define
\begin{align*}
v^\star(i^\star)=\argmax_{v\in[0,1]^K}\left\{\sum_{i=1}^Kp_iv_i-v_{i^\star}-\frac{\gamma}{4}\sum_{i=1}^Kp_i\sum_{j=1}^Kg(x,i,j)(f(x,j)-v_j)^2\right\}.
\end{align*}
Direct calculation shows that $v_j^\star(i^\star)=\max\left\{f(x,j)-\frac{(p_j-\mathbbm{1}\{j=i^\star\})^2}{\gamma\sum_{i=1}^Kp_ig(x,i,j)}, 0\right\}$ for all $j\in [K]$.
Let $h_{i^\star}(p)$ be the maximum value attained by $v^\star(i^\star)$, which is convex in $p$ since it is a point-wise supremum over functions linear in $p$.
It is then clear that solving $\argmin_{p\in\Delta([K])}\overline{\dec}_{\gamma}(p;f,g,x)$ is equivalent to solving the following constrained convex problem:
\begin{align*}
    &\min_{u\in\R,p\in\Delta([K])}~~~~~ u\\
    s.t.~~~~~~& h_{i^\star}(p)\leq u, ~\forall i^\star\in[K].
\end{align*}

\subsection{Omitted Details in \pref{sec:experiment}}\label{app:exp}
In this section, we include the omitted details for our experiments.

\paragraph{Dataset Details.} For the synthetic datasets, as mentioned in \pref{sec:syn}, the two datasets differ in how $\{x_t\}_{t=1}^T$ are generated. In the first dataset, each coordinate of $x_t\in\R^{32}$ is independently drawn from $\calN(0,1)$; in the second dataset, the first $8$ coordinates of $x_t$ is independently drawn from $\calN(0,1)$ and the remaining coordinates are all $1$. The real auction dataset we used in \pref{sec:real} is an eBay auction dataset (available at \url{https://cims.nyu.edu/~munoz/data/}) with the $t$-th datapoint consisting of a 78-dimensional feature vector $x_t$, a winning price of the auction $v_t$, and a competing price $w_t$. We treat the winning price as the value of the learner in our experiment. We randomly select a subset of $5000$ data points whose winning price is in range $[100,300]$ and normalize the value and the competing price to range $[0,1]$.

\paragraph{Model Details.} We implement the graph oracle as a linear classification model, aiming to predict the distribution of the competing price, denoted as $p_w(x_t)\in \Delta([K])$. With $p_w(x_t)$, we sample $\widehat w_t \sim p_w(x_t)$ and the predicted graph $g_t$ is calculated as
\begin{align}\label{eq:calc_gt}
g_t(x_t,i,j)=\mathbbm{1}[i \ge \widehat w_t] \cdot \mathbbm{1}[i \le j] +\mathbbm{1}[i < \widehat w_t] \cdot \mathbbm{1}[j < \widehat w_t].
\end{align}
We implement the loss oracle as a two-layer fully connected neural network with hidden size $32$. The neural network predicts the value of the data point $x_t$, denoted as $\widehat v_t$. The predicted loss of each arm $i$ is then calculated as:
\begin{align}\label{eqn:ft_exp}
f_t(x_t,i)=\frac{1}{2}\left[1-\mathbbm{1}\left[a_i \ge \widehat w_t\right] \cdot (\widehat v_t - a_i)\right],
\end{align}
where $a_i=(i-1)\epsilon$.
\paragraph{Training Details.} For the graph oracle, the loss function of each round $t$ is calculated as: $\frac{1}{K}\sum_{(i,j,b) \in S_t} \ell_{\log}(\E_{\wh{w}_t\sim p_w(x_t)}[g_t(x_t,i,j)], b)$,
where $S_t$ is the input dataset defined in \graphCBp. 
For the loss oracle, the loss function is calculated as $\frac{1}{|A_t|}\sum_{j \in A_t} (f_t(x_t,j)-\ell_{t,j})^2.$

We apply online gradient descent to train both models. Since the loss regression model aims to predict the value, we only update it when the learner wins the auction and observes the value. For experiments on the real auction dataset, learning rate is searched over $\{0.005,0.01,0.05\}$ for the loss oracle and over $\{0.01,0.05\}$ for the graph regression oracle. For experiments on the synthetic datasets, they are searched over $\{0.005,0.01,0.02\}$ and $\{0.01,0.05\}$ respectively. For \squareCB, we set the exploration parameter $\gamma=c \cdot \sqrt{KT}$ (based on what its theory suggests), where $c$ is searched over $\{0.5,1,2\}$. For our \graphCBp, we set $\gamma=c \cdot \sqrt{T}$, where $c$ is also searched over $\{0.5,1,2\}$. The experiment on the real auction dataset is repeated with $8$ different random seeds and the experiment on the synthetic datasets is repeated with $4$ different random seeds.

\paragraph{A Closed-Form Solution.} In this part, we introduce a closed form of $p_t$ which leads to a more efficient implementation of \graphCBp for the specific setting considered in the experiments. Specifically, given the predicted competing price $\wh{w}_t$, let $b\in[\frac{1}{\epsilon}+1]$ be the smallest action such that $(b-1)\eps \ge \widehat w_t$. Given $f_t$ and $g_t$ defined in \pref{eqn:ft_exp} and \pref{eq:calc_gt}, define a closed-form $p$ which concentrates on action $1$ (bidding $0$) and action $b$ as follows,
\begin{align}\label{eq:closed_form}
p_1=\begin{cases}
\frac{1}{2+\gamma\pa{1/2-f_t(x_t,b)}} &  f_t(x_t,b) \leq \frac{1}{2}\\
1-\frac{1}{2+\gamma\pa{f_t(x_t,b)-1/2}} & f_t(x_t,b) >\frac{1}{2}
\end{cases}, \quad p_b=1-p_1,
\end{align}
In the following lemma, we prove that the closed-form probability distribution in \pref{eq:closed_form} guarantees $\dec_\gamma(p;f_t,g_t,x_t) \le \frac{4}{\gamma}$,
which is enough for all our analysis to hold (despite the fact that it does not exactly minimize the DEC).
\begin{lemma}
For any $x_t\in\calX$, $f_t \in \calF$, and
$g_t$ in the form of \pref{eq:calc_gt}, the probability distribution $p$ defined in \pref{eq:closed_form}
guarantees $\dec_\gamma(p,f_t,g_t,x_t) \le \frac{4}{\gamma}$.
\end{lemma}
\begin{proof}
According to the analysis in \pref{lem:main-dec-lemma}, it suffices to bound $\overline{\dec}_\gamma(p,f_t,g_t,x_t)$. Based on \pref{eqn:exxx}, we have
\begin{align*}
\overline{\dec}_\gamma(p,f_t,g_t,x_t) &= \max_{i^\star \in [K]} \left\{\sum_{i=1}^Kp_{i}f_t(x_t,i) - f_t(x_t,i^\star) +\frac{1}{\gamma}\sum_{i\neq i^\star}\frac{p_i^2}{\sum_{i'=1}^Kp_{i'}g_t(x_t,i',i)} + \frac{1}{\gamma}\frac{(1-p_{i^\star})^2}{\sum_{i'=1}^Kp_{i'}g_t(x_t,i',i^\star)}\right\} \\
&\le \max_{i^\star \in [K]} \left\{\sum_{i=1}^Kp_{i}f_t(x_t,i) - f_t(x_t,i^\star) +\frac{1}{\gamma}\sum_{i=1}^K\frac{p_i^2}{\sum_{i'=1}^Kp_{i'}g_t(x_t,i',i)} + \frac{1}{\gamma\sum_{i'=1}^K p_{i'}g_t(x_t,i',i^\star)} \right\}.
\end{align*}
Note that according to \pref{eqn:ft_exp} and the definition of $b$, for $i<b$, $f_t(x,i)=\frac{1}{2}$. We now first consider the case $f_t(x_t,b) \le \frac{1}{2}$ and $p_1=\frac{1}{2+\gamma(\frac{1}{2}-f_t(x_t,b))} \le \frac{1}{2}$.
\begin{enumerate}
\item Suppose $1 \le i^\star < b$, we observe that $f_t(x_t,i^\star)=\frac{1}{2}$ and have
\begin{align*}
\overline{\dec}_\gamma(p,f_t,g_t,x_t)& \le \frac{p_1}{2} + p_b f_t(x_t,b) - \frac{1}{2} + \frac{1}{\gamma}\pa{\frac{1}{p_1}+1} \\
&=\frac{p_1}{2}+(1-p_1)f_t(x_t,b)-\frac{1}{2}+\frac{1}{\gamma}\pa{2+\gamma\pa{\frac{1}{2}-f_t(x_t,b)}+1} \\
&=\frac{p_1}{2}-p_1f_t(x_t,b)+\frac{3}{\gamma} \\
&=p_1\pa{\frac{1}{2}-f_t(x_t,b)}+\frac{3}{\gamma} \\
&\le \frac{4}{\gamma}.
\end{align*} 
\item Suppose $b \le i^\star \le K$. According to~\pref{eqn:ft_exp}, we know that $f_t(x_t,i^\star) \ge f_t(x_t,b)$ and obtain
\begin{align*}
\overline{\dec}_\gamma(p,f_t,g_t,x_t)& \le \frac{p_1}{2} + p_b f_t(x_t,b) - f_t(x_t,b) + \frac{1}{\gamma}\pa{\frac{1}{p_b}+1} \\
&\le \frac{p_1}{2}+(1-p_1)f_t(x_t,b)-f_t(x_t,b)+\frac{3}{\gamma} \tag{since $p_b =1-p_1\ge \frac{1}{2}$}\\
&=p_1 \pa{\frac{1}{2}-f_t(x_t,b)}+\frac{3}{\gamma} \\
&\leq\frac{4}{\gamma}.
\end{align*}
\end{enumerate}
Then we consider the case when $f_t(x_t,b) > \frac{1}{2}$ and $p_1=1-\frac{1}{2+\gamma(f_t(x_t,b)-\frac{1}{2})} > \frac{1}{2}$.
\begin{enumerate}
\item Suppose $1 \le i^\star < b$, we have
\begin{align*}
\overline{\dec}_\gamma(p,f_t,g_t,x_t)& \le \frac{p_1}{2} + p_b f_t(x_t,b) - \frac{1}{2} + \frac{1}{\gamma}\pa{\frac{1}{p_1}+1} \\
&=\frac{1}{2}+p_b\pa{f_t(x_t,b)-\frac{1}{2}}-\frac{1}{2}+\frac{1}{\gamma}\pa{\frac{1}{p_1}+1} \\
&\le p_b\pa{f_t(x_t,b)-\frac{1}{2}} + \frac{3}{\gamma} \tag{since $p_1 > \frac{1}{2}$}\\
&\le \frac{4}{\gamma}.
\end{align*} 
\item Suppose $b \le i^\star \le K$, we have
\begin{align*}
\overline{\dec}_\gamma(p,f_t,g_t,x_t)& \le \frac{p_1}{2} + p_b f_t(x_t,b) - f_t(x_t,b) + \frac{1}{\gamma}\pa{\frac{1}{p_b}+1} \\
&= \frac{1}{2}+p_b\pa{f_t(x_t,b)-\frac{1}{2}}-f_t(x_t,b)+\frac{1}{\gamma}\pa{2+\gamma\pa{f_t(x_t,b)-\frac{1}{2}}+1} \\
&=p_b\pa{f_t(x_t,b)-\frac{1}{2}} + \frac{3}{\gamma} \\
&\le \frac{4}{\gamma}.
\end{align*}
\end{enumerate}
Combining the two cases finishes the proof.
\end{proof}

%%%%%%%%%%%%%%%%%%%%%%%%%%%%%%%%%%%%%%%%%%%%%%%%%%%%%%%%%%%%%%%%%%%%%%%%%%%%%%%
%%%%%%%%%%%%%%%%%%%%%%%%%%%%%%%%%%%%%%%%%%%%%%%%%%%%%%%%%%%%%%%%%%%%%%%%%%%%%%%

\end{document}